\documentclass{article}
\PassOptionsToPackage{square,comma,numbers,sort}{natbib}  
\usepackage[preprint]{neurips_2021}
\usepackage[utf8]{inputenc} 
\usepackage[T1]{fontenc}    
\usepackage{hyperref}
\hypersetup{colorlinks,allcolors=blue}
\usepackage{url}            
\usepackage{booktabs}       
\usepackage{amsfonts}       
\usepackage{nicefrac}       
\usepackage{microtype}      
\usepackage[svgnames]{xcolor}         
\usepackage{bbm}
\usepackage{pifont}
\newcommand{\cmark}{\ding{51}}
\newcommand{\xmark}{\ding{55}}
\title{Implicit Sparse Regularization: \\ The Impact of Depth and Early Stopping}

\author{%
  Jiangyuan Li\\
    \texttt{jiangyuanli@tamu.edu} \\
    Texas A\&M University\\
   \And
  Thanh V. Nguyen \\
  \texttt{thanhng.cs@gmail.com} \\
  AWS AI \\
   \AND
  Chinmay Hegde \\
  \texttt{chinmay.h@nyu.edu} \\
  New York University\\
   \And
  Raymond K. W. Wong \\
  \texttt{raywong@tamu.edu} \\
  Texas A\&M University\\
}

\usepackage{amsmath, amsfonts, amssymb, amsthm, mathtools}
\usepackage{bm}

\newtheorem{theorem}{Theorem}
\newtheorem{definition}{Definition}
\newtheorem{proposition}{Proposition}
\newtheorem{lemma}{Lemma}
\newtheorem{corollary}{Corollary}
\newtheorem{remark}{Remark}

\newcommand{\norm}[1]{\left\lVert#1\right\rVert}

\renewcommand{\vec}[1]{\bm{\mathrm{#1}}}
\newcommand{\wstar}{\ensuremath{\vec{w}^{\star}}}

\newcommand{\id}[1]{\ensuremath{\vec{1}_{#1}}}

\renewcommand{\matrix}[1]{\ensuremath{\mathbf{#1}}}
\newcommand{\matrixid}{\ensuremath{\matrix{I}}}
\newcommand{\X}{\ensuremath{\matrix{X}}}
\newcommand{\Xt}{\ensuremath{\X^{\mathsf{T}}}}
\newcommand{\XtX}{\ensuremath{\X^{\mathsf{T}}\X}}
\newcommand{\wmax}{\ensuremath{w^{\star}_{\max}}}
\newcommand{\wmin}{\ensuremath{w^{\star}_{\min}}}

\newcommand{\R}{\mathbb{R}}
\newcommand{\xxi}{\vec{\xi}}

\newcommand{\revise}[1]{\textcolor{black}{#1}}

\begin{document}

\maketitle

\begin{abstract}
  In this paper, we study the implicit bias of gradient descent for sparse regression. We extend results on regression with quadratic parametrization, which amounts to depth-2 diagonal linear networks, to more general depth-$N$ networks, under more realistic settings of noise and correlated designs. We show that early stopping is crucial for gradient descent to converge to a sparse model, a phenomenon that we call \emph{implicit sparse regularization}. This result is in sharp contrast to known results for noiseless and uncorrelated-design cases. 
  We characterize the impact of depth and early stopping and show that for a general depth parameter $N$, gradient descent with early stopping achieves minimax optimal sparse recovery with sufficiently small initialization $w_0$ and step size $\eta$. In particular, we show that increasing depth enlarges the scale of working initialization and the early-stopping window so that this implicit sparse regularization effect is more likely to take place.
\end{abstract}

\section{Introduction}

\textbf{Motivation.} Central to recent research in learning theory is the insight that the choice of optimization algorithms plays an important role in model generalization \cite{zhang2016understanding, hardt2016train, li2020gradient}.
A widely adopted view is that (stochastic) gradient descent --- the most popular optimization algorithm in machine learning --- exhibits some implicit form of regularization. Indeed for example, in the classical under-determined least squares setting, gradient descent (with small step size) starting from the origin converges to the model with minimum Euclidean norm. Similar implicit biases are also observed in deep neural network training in which the networks typically have many more parameters than the sample size. There, gradient descent without explicit regularization finds solutions that not only interpolate the training data points but also generalize well on test sets \cite{hardt2016train,neyshabur2017geometry, soudry2018separable, belkin2019reconciling, muthukumar2020harmless}. 

This insight, combined with the empirical success stories of deep learning, has sparked significant interest among theoretical researchers to rigorously understand implicit regularization. The majority of theoretical results focus on well-understood problems such as regression with linear models \cite{saxe2013exact,  gunasekar2018implicitCNN, gunasekar2018biasoptgem, vaskevicius2019implicit, zhao2019implicit, gissin2019implicit} and matrix factorization \cite{gidel2019implicit, gunasekar2018implicit, li2018algorithmic, arora2019implicit}, and show that the parametrization (or architecture) of the model plays a crucial role. For the latter, Gunasekar et al.~\cite{gunasekar2018implicit} conjectured that gradient descent on factorized matrix representations converges to the solution with minimum nuclear norm. The conjecture was partially proved by Li et al.~\cite{li2018algorithmic} under the Restricted Isometry Property (RIP) and the absence of noise. 
Arora et al.~\cite{arora2019implicit} further show the same nuclear-norm implicit bias using depth-$N$ linear networks (i.e., the matrix variable is factorized into $N$ components).

Parallel work on nonlinear models and classification \cite{soudry2018separable, gunasekar2018implicitCNN} has shown that gradient descent biases the solution towards the max-margin/minimum $\ell_2$-norm solutions over separable data. 
The scale of initialization in gradient descent leads to two learning regimes (dubbed ``kernel'' and ``rich'') in linear networks \cite{woodworth2020kernelregimes}, shallow ReLU networks \cite{williams2019gradient} and deep linear classifiers \cite{moroshko2020implicit}.
\revise{
Li et al.~\cite{li2020towards} showed that depth-2 network requires an exponentially small initialization, whereas depth-$N$ network $(N\geq3)$ only requires a polynomial small initialization, to obtain low-rank solution in matrix factorisation. 
Woodworth et al.~\cite{woodworth2020kernelregimes} obtained a similar result for high dimensional sparse regression.
}

The trend in the large majority of the above works has been to capture implicit regularization of gradient descent using some type of norm with respect to the working parametrization \cite{woodworth2020kernelregimes, neu2018iterate, raskutti2014early, suggala2018connecting}. 
{On the other hand, progress on understanding the \emph{trajectory} of gradient descent has been somewhat more modest.} \cite{vaskevicius2019implicit, zhao2019implicit} study the sparse regression problem using quadratic and Hadamard parametrization respectively and show that gradient descent with small initialization and careful \emph{early stopping} achieves minimax optimal rates for sparse recovery.
Unlike \cite{li2018algorithmic, woodworth2020kernelregimes} that study noiseless settings and require no early stopping, \cite{vaskevicius2019implicit, zhao2019implicit} mathematically characterize the role of early stopping and empirically show that it may be necessary to prevent gradient descent from over-fitting to the noise. These works suggest that the inductive bias endowed by gradient descent may be influenced not only by the choice of parametrization, \emph{but also algorithmic choices} such as initialization, learning rate, and the number of iterations. However, our understanding of such gradient dynamics is incomplete, \emph{particularly} in the context of deep architectures; see Table~\ref{tbl:comparison} for some comparisons. 

\textbf{Contributions.} Our focus in this paper is the implicit regularization of (standard) gradient descent for high dimensional sparse regression, namely \emph{implicit sparse regularization}. Let us assume a ground-truth sparse linear model and suppose we observe $n$ noisy samples $(\vec{x}_i, y_i)$, such that  $\vec{y} = \X\vec{w}^* + \vec{\xi}$; a more formal setup is given in Section \ref{sec:setup}. Using the samples, we consider gradient descent on a squared loss $\| \X \vec{w} - \vec{y}\|^2$ with no explicit sparsity regularization. Instead, we write the parameter vector $\vec{w}$ in the form 
$\vec{w} = \vec{u}^N - \vec{v}^N$ with $N \ge 2$. Now, the regression function $f(\vec{x}, \vec{u}, \vec{v}) = \langle \vec{x}, \vec{u}^N - \vec{v}^N \rangle$ can be viewed as a depth-$N$ \emph{diagonal linear network}~\cite{woodworth2020kernelregimes}. Minimizing the (now non-convex) loss over $\vec{u}$ and $\vec{v}$ with gradient descent is then analogous to training this depth-$N$ network. 

Our main contributions are the following. We characterize the impact of both the depth and early stopping for this non-convex optimization problem. Along the way, we also generalize the results of \cite{vaskevicius2019implicit} for $N > 2$. We show that under a general depth parameter $N$ and an incoherence assumption on the design matrix, gradient descent with early stopping achieves minimax optimal recovery with sufficiently small initialization $\vec{w}_0$ and step size $\eta$. The choice of step size is of order $O(1/N^2)$. Moreover, the upper bound of the initialization, as well as the early-stopping window, increase with $N$, suggesting that depth leads to \revise{a more accessible generalizable solution on gradient trajectories.}

\textbf{Techniques.} At a high level, our work continues the line of work on implicit bias initiated in \cite{vaskevicius2019implicit, zhao2019implicit, woodworth2020kernelregimes} and extends it to the deep setting. Table \ref{tbl:comparison} highlights key differences between our work and \cite{vaskevicius2019implicit, gissin2019implicit, woodworth2020kernelregimes}. Specifically, \revise{Woodworth et al. \cite{woodworth2020kernelregimes} study the \emph{interpolation} given by the gradient flow of the squared-error loss function.}
Vaskevicius et al. \cite{vaskevicius2019implicit} analyze the finite gradient descent and characterize the implicit sparse regularization \revise{on the \emph{recovery} of true parameters} with $N=2$. Lastly, Gissin et al. \cite{gissin2019implicit} discover the incremental learning dynamic of gradient flow for general $N$ but in an idealistic model setting where $\vec{u} \succeq 0, \vec{v} = 0$, $\vec{\xi} = 0$, uncorrelated design and with infinitely many samples.

\begin{table}[ht!]
\label{tbl:comparison}
\caption{Comparisons with closely related recent work. GF/GD: gradient flow/descent, respectively.}
\centering
\resizebox{\columnwidth}{!}{%
\begin{tabular}{ c|c|c|c|c|c|c } 
  & Design Matrix & In Noise & Depth & Early Stopping & GD vs. GF & \revise{Remark}\\ 
 \toprule
Vaskevicius et al. (2020) \cite{vaskevicius2019implicit} & RIP & \cmark & $N=2$ & \cmark & GD & recovery \\ 
Gissin et al. (2020) \cite{gissin2019implicit} & uncorrelated & \xmark & $N=2, N>2$ & \xmark & GF & interpolation\\ 
 Woodworth et al. (2020) \cite{woodworth2020kernelregimes} & \revise{\xmark} & \xmark & $N=2, N>2$ &\xmark & GF & interpolation\\
 This paper & $\mu$-coherence & \cmark & $N>2$ &\cmark & GD & recovery\\
 \bottomrule
\end{tabular}%
}
\end{table}
At first glance, one could attempt a straightforward extension of the proof techniques in \cite{vaskevicius2019implicit} to general settings of $N > 2$. However, this turns out to be very challenging. Consider even the simplified case where the true model $\wstar$ is non-negative, the design matrix is unitary (i.e., $n^{-1} \XtX = \matrixid$), and the noise is absent ($\vec{\xi} = 0$); this is the setting studied in~\cite{gissin2019implicit}. For each entry $w_i$ of $\vec{w}$, the $t^{\textrm{th}}$ iterate of gradient descent over the depth-$N$ reparametrized model is given by: 
\[
w_{i,t+1} = w_{i,t}
\left(1+w_{i,t}^{1-\frac2N}(w^\star_i - w_{i,t})
\right)^N,
\]
which is no longer a simple multiplicative update. 
As pointed out in \cite{gissin2019implicit} (see their Appendix C), the recurrence relation is not analytically solvable due to the presence of the (pesky) term $w_{i,t}^{1-\frac2N}$ when $N>2$.
Moreover, this extra term $w_{i,t}^{1-\frac2N}$ leads to widely divergent growth rates of weights with different magnitudes, which further complicates analytical bounds.
To resolve this and rigorously analyze the dynamics for $N>2$, we rely on a novel first order, continuous approximation to study growth rates without requiring additional assumptions on gradient flow, and carefully bound the approximation error due to finite step size; see Section~\ref{sec:proofingred}.

\section{Setup}
\label{sec:setup}

\textbf{Sparse regression/recovery.} Let $\wstar\in\R^p$ be a $p$-dimensional sparse vector  with $k$ non-zero entries. Assume that we observe $n$ data points $(\vec{x}_i, y_i)\in\R^p\times\R$ such that $y_i=\langle \vec{x}_i,\wstar\rangle + \xi_i$ for $i=1,\ldots,n$, where $\vec{\xi} = (\xi_1, \ldots, \xi_n)$ is the noise vector. We do not assume any particular scaling between the number of observations $n$ and the dimension $p$. Due to the sparsity of $\wstar$, however, we allow $n  \ll p$.

The linear model can be expressed in the matrix-vector form:
\begin{equation}
\vec{y} = \X\wstar + \vec{\xi},
\label{eqn:lm}
\end{equation}
with the $n\times p$ design matrix $\X=[\vec{x}_1^\top,\ldots,\vec{x}_n^\top]^\top$, where $\vec{x}_i$ denotes the $i^{\textrm{th}}$ \emph{row} of $\X$. 
We also denote $\X=[\X_1,\ldots,\X_p]$, where $\X_i$ denotes the $i^{\textrm{th}}$ \emph{column} of $\X$. 

The goal of sparse regression is to estimate the unknown, sparse vector $\wstar$ from the observations. Over the past two decades, this problem has been a topic of active research in statistics and signal processing \cite{tibshirani1996lasso}. A common approach to sparse regression is penalized least squares with sparsity-induced regularization such as $\ell_0$ or $\ell_1$ penalties/constraints, leading to several well-known estimators \cite{tibshirani1996lasso, chen2001bp, candes2007dantzig} and algorithms \cite{bredies2008ista, agarwal2012fast}. Multiple estimators enjoy optimal statistical and algorithmic recovery guarantees under some conditions of  the design matrix $\X$ (e.g., RIP \cite{candes2005decoding}) and the noise $\vec{\xi}$.

We deviate from the standard penalized least squares formulation and 
instead learn $\vec{w}^*$ via a polynomial parametrization:
\[
\vec{w} = \vec{u}^N-\vec{v}^N, \quad \vec{u},\vec{v}\in\R^p,
\]
where $N\ge2$ and $\vec{z}^N=[z_1^N,\dots, z_p^N]^\top$ for any $\vec{z}=[z_1,\dots,z_N]^\top\in\R^p$.
The regression function $f(\vec{x}, \vec{u}, \vec{v}) = \langle \vec{x}, \vec{u}^N - \vec{v}^N \rangle$ induced by such a parametrization is equivalent to a $N$-layer diagonal linear network \cite{woodworth2020kernelregimes} with $2p$ hidden neurons and the diagonal weight matrix shared across all layers.

 Given the data $\{\X, \vec{y}\}$ observed  in \eqref{eqn:lm}, we analyze gradient descent with respect to the new parameters $\vec{u}$ and  $\vec{v}$ over the mean squared error loss without explicit regularization:
\[
\mathcal{L}(\vec{u},\vec{v})
=
\frac1n
\norm{\X(\vec{u}^N - \vec{v}^N)-\vec{y}}_2^2, \quad \vec{u},\vec{v}\in\R^p.
\]

Even though the loss function yields the same value for the two parametrizations, $\mathcal{L}(\vec{u},\vec{v})$ is non-convex in $\vec{u}$ and $\vec{v}$. Unlike several recent studies in implicit regularization for matrix factorization and regression \cite{li2018algorithmic, woodworth2020kernelregimes, gissin2019implicit}, we consider the noisy setting, which is more realistic and leads to more insights into the bias induced during the optimization.
Because of noise, the loss evaluated at the ground truth (i.e., any $\vec{u}, \vec{v}$ such that $\wstar=\vec{u}^N-\vec{v}^N$) is not necessarily zero or even minimal.

\textbf{Gradient descent.} The standard gradient descent update over $\mathcal{L}(\vec{u},\vec{v})$ reads as:
\begin{align}
    &\vec{u}_0=\vec{v}_0=\alpha\id{}, \nonumber \\ 
    &(\vec{u}_{t+1},\vec{v}_{t+1})
    = 
    (\vec{u}_t, \vec{v}_t)
    - 
    \eta \frac{\partial \mathcal{L}(\vec{u}_t,\vec{v}_t)}{\partial(\vec{u}_t, \vec{v}_t)}, \quad t = 0, 1, \ldots.
    \label{eq:gd}
\end{align}
Here, $\eta>0$ is the step size and $\alpha>0$ is the initialization of $\vec{u}, \vec{v}$. 
In general, we analyze the algorithm presented in \eqref{eq:gd}, and at each step $t$, we can estimate the signal of interest by simply calculating  $\vec{w}_t = \vec{u}_t^N-\vec{v}_t^N$.
\revise{
We consider constant initialization for simplicity sake. Our results apply for random initialization concentrating on a small positive region with a probabilistic statement. 
}

Vaskevicius et al. \cite{vaskevicius2019implicit} establish the implicit sparse regularization of gradient descent for $N=2$ and show minimax optimal recovery, provided sufficiently small $\alpha$ and early stopping. Our work aims to generalize that result to $N > 2$ and characterize the role of $N$ in convergence. 

\textbf{Notation.} We define $S=\{i\in\{1,\ldots,p\}:w^\star_i\neq0\}$ and $S^c =\{1,\ldots,p\}\backslash S$. The largest and smallest absolute value on the support is denoted as $\wmax = \max_{i\in S} |w^\star_i|$ and $\wmin = \min_{i\in S} |w^\star_i|$. We use $\id{}$ to denote the vector of all ones and $\id{S}$ denotes the vector whose elements on $S$ are all one and 0 otherwise.
Also, $\odot$ denotes coordinate-wise multiplication.
We denote $\vec{s}_t =\id{S}\odot \vec{w}_t$ and $\vec{e}_t=\id{S^c} \odot \vec{w}_t$ meaning the signal part and error part at each time step $t$. 
We use $\wedge$ and $\vee$ to denote the pointwise maximum and minimum.
The coordinate-wise inequalities are denoted as $\succcurlyeq$.
We denote inequalities up to multiplicative absolute constants by $\lesssim$, which means that they do not depend on any parameters of the problem.

\begin{definition}
Let $\X\in\mathbb{R}^{n\times p}$ be a matrix with $\ell_2$-normalized columns $\X_1,\ldots,\X_p$, i.e., $\norm{\X_i}_2=1$ for all $i$. The coherence $\mu=\mu(\X)$ of the matrix $\X$ is defined as 
\[
\mu \coloneqq \max_{1\leq i \neq j \leq p}|\langle \X_i,\X_j \rangle|.
\]
The matrix $\X$ is said to be satisfying $\mu$-incoherence.
\end{definition}

The coherence is a measure for the suitability of the measurement matrix in compressive sensing \cite{foucart2013invitation}. In general, the smaller the coherence, the better the recovery algorithms perform. 
\revise{There are multiple ways to construct a sensing matrix with low-incoherence. One of them is based on the fact that sub-Gaussian matrices satisfy low-incoherence property with high probability~\cite{carin2011coherence, donoho2005stable}.}
In contrast to the coherence, the Restricted Isometry Property (RIP) is a powerful performance measure for guaranteeing sparse recovery and has been widely used in many contexts. However, verifying the RIP for deterministically constructed design matrices is NP-hard. On the other hand, coherence is a computationally tractable measure and its use in sparse regression is by now classical \cite{donoho2005stable, candes2011compressed}. Therefore, in contrast with previous results~\cite{vaskevicius2019implicit} (which assumes RIP), the assumptions made in our main theorems are verifiable in polynomial time.

\section{Main Results}
\label{sec:main-res}
We now introduce several quantities that are relevant for our main results. First, the condition number $r\coloneqq\wmax/\wmin$ plays an important role when we work on the incoherence property of the design matrix. Next, we require an upper bound on the initialization $\alpha$, which depends on the following terms: 
\begin{gather*}
\Phi(\wmax,\wmin,\epsilon,N) \coloneqq
\left(\frac18\right)^{2/(N-2)} 
\wedge
\left(\frac{(\wmax)^{(N-2)/N}}{\log \frac{\wmax}{\epsilon}} \right)^{2/(N-2)}
\wedge 
\left(\frac{(\wmin)^{(N-2)/N}}{ \log \frac{\wmin}{\epsilon}}\right)^{4/(N-2)},
\\
\Psi(\wmin,N) \coloneqq
(2-2^\frac{N-2}{N})^\frac{1}{N-2}(\wmin)^{\frac1N}
\wedge 
2^{\frac3N}(2^\frac{1}{N}-1)^\frac{1}{N-2} (\wmin)^{\frac1N}.
\end{gather*}
Finally, define
\[
\zeta 
\coloneqq 
\frac15 \wmin 
\vee 
\frac{200}{n}\norm{\Xt\vec{\xi}}_\infty 
\vee 
200\epsilon.
\]
We are now ready to state the main theorem:

\begin{theorem}
Suppose that $k\geq 1$ and $\X/\sqrt{n}$ satisfies $\mu$-incoherence with $\mu\lesssim 1/k r$. Take any precision $\epsilon >0$, and let the initialization be such that
\begin{equation}
0<\alpha \leq \left(\frac{\epsilon}{p+1}\right)^{4/N}
\wedge \Phi(\wmax,\wmin,\epsilon,N) 
\wedge \Psi(\wmin,N).
\label{eqn:init_upper}
\end{equation}
For any iteration $t$ that satisfies
\begin{equation}
T_l(\wstar,\alpha,N,\eta,\zeta,\epsilon)
\le t \le
T_u(\wstar,\alpha,N,\eta,\zeta,\epsilon),
\label{eqn:iteration_bound}
\end{equation}
where $T_l(\cdot)$ and $T_u(\cdot)$ are given in \eqref{eq:Tl-and-Tu} of the Appendix,
the gradient descent algorithm \eqref{eq:gd} with step size $\eta\leq \frac{\alpha^N}{8N^2(\wmax)^{(3N-2)/N}}$ yields the iterate $\vec{w}_t$ with the following property:
\begin{equation}
|w_{t,i}-w_i^\star|\lesssim 
\begin{cases}
    \norm{\frac1n \Xt\vec{\xi}}_\infty \vee \epsilon \quad &\text{if }i\in S \text{ and }\wmin\lesssim\norm{\frac1n \Xt\vec{\xi}}_\infty \vee \epsilon, \\
    \left|\frac1n (\Xt\vec{\xi})_i\right|
    \vee
    k\mu\norm{\frac1n \Xt\vec{\xi}\odot \id{S}}_\infty 
    \vee
    \epsilon
    \quad &\text{if }i\in S \text{ and }\wmin \gtrsim\norm{\frac1n \Xt\vec{\xi}}_\infty \vee \epsilon, \\
    
    \alpha^{N/4}\quad &\text{if }i\notin S.
\end{cases}
\label{eqn:main_guarantee}
\end{equation}
In the special case $\wstar= \vec{0}$, 
if $\alpha \leq \left(\frac{\epsilon}{p+1}\right)^{4/N}$, $\eta\leq \frac{1}{N(N-1)\zeta\alpha^{(N-2)/2}}$ and $t\le T_u(\wstar,\alpha,N,\eta,\zeta,\epsilon)$, then we have $|w_{t,i}-w_i^\star|\leq \alpha^{N/4},\forall i$.
\label{thm:general}
\end{theorem}
Theorem \ref{thm:general} states the convergence of the gradient descent algorithm \eqref{eq:gd} in $\ell_\infty$-norm.
\revise{
The exact formula of $T_l(\cdot)$ and $T_u(\cdot)$ is omitted here due to the space limitation.
We ensure that $T_u(\cdot)>T_l(\cdot)$ so that there indeed exists some epochs to early stop at.
}
The error bound on the signal is invariant to the choice of $N\ge 2$, and the overall bound generalizes that of \cite{vaskevicius2019implicit} for $N = 2$. We also establish the convergence result in $\ell_2$-norm in the following corollary:

\begin{corollary}
Suppose the noise vector $\vec{\xi}$ has independent $\sigma^2$-sub-Gaussian entries and $\epsilon=2\sqrt{\frac{\sigma^2\log(2p)}{n}}$. Under the assumptions of Theorem \ref{thm:general}, the gradient descent algorithm \eqref{eq:gd} would produce iterate $\vec{w}_t$ satisfying $\norm{\vec{w}_t - \wstar}_2^2\lesssim (k\sigma^2 \log p)/n$ with probability at least $1-1/(8p^3)$. 
\label{cor:l2-error}
\end{corollary}

\revise{
Note that the error bound we obtain is minimax-optimal, which is the same as \cite{vaskevicius2019implicit} in the $N=2$ case. However, with some calculation, the sample complexity we obtain here is $n\gtrsim k^2 r^2$, while the sample complexity in \cite{vaskevicius2019implicit} is $n\gtrsim k^2\log^2r \log p/k$. Although neither our work nor \cite{vaskevicius2019implicit} achieved the optimal sample complexity $k\log p/k$, the goal of this work is to understand how the depth parameter affects implicit sparse regularization.}

Let us now discuss the implications of Theorem \ref{thm:general} and the role of initialization and early stopping:

\textbf{(a) Requirement on initialization.} 
To roughly understand the role of initialization and the effect of $N$, we look at the non-negative case where $\wstar \succcurlyeq 0$ and $\vec{w}= \vec{u}^N$. This simplifies our discussion while still  capturing the essential insight of the general setting. At each step, the ``update'' on $\vec{w}$ can be translated from the corresponding gradient update of $\vec{u} = \vec{w}^{1/N}$ as
\begin{equation}
\begin{aligned}
  \vec{w}_0
  &= \alpha^N \id{}, \\
 \vec{w}_{t+1} 
 &= \vec{w}_t\odot\left(\id{}-\frac{2N\eta}{n}\biggl(
 \Xt\X (\vec{w}_t - \vec{w}^*) -  \Xt\vec{\xi}
 \biggr)\odot \vec{w}_t^{(N-2)/N} \right)^N. 
\end{aligned}
\label{eq:non-negative-updates}
\end{equation}

In order to guarantee the convergence, we require the initialization $\alpha$ to be sufficiently small so the error outside the support can be controlled. On the other hand, too small initialization slows down the convergence of the signal. Interestingly, the choice of $N$ affects the allowable initialization $\alpha$ that results in guarantees on the entries inside and outside the support.

Specifically, the role of $N$ is played by the term $\vec{w}_t^{(N-2)/N}$ in \eqref{eq:non-negative-updates}, which simply disappears as $N=2$. 
Since this term only affects the update of $\vec{w}_{t+1}$ entry-wise, we only look at a particular entry $w_t$ of $\vec{w}_t$. Let $w_t$ represent an entry outside the support.
For $N>2$, the term $w_t^{(N-2)/N}$ is increasingly small as $N$ increases and ${w}_t < 1$.
Therefore, with a small initialization, it remains true that $w_t<1$ for the early iterations. 
 Intuitively, this suggests that the requirement on the upper bound of the initialization would become looser when $N$ gets larger. This indeed aligns with the behavior of the upper bound we derive in our theoretical results. 
 Since $\alpha=w_0^{1/N}$ increases naturally with $N$, we fix $w_0=\alpha^N$ instead of $\alpha$ to mimic the same initialization in terms of $w_0$, for the following comparison. 
 
 We formalize this insight in Theorem \ref{thm:non-negative} in Appendix \ref{sec:proof-for-non-negative} and show the convergence of \eqref{eq:non-negative-updates} under the special, non-negative case. Note that, in terms of initialization requirement, the only difference from Theorem \ref{thm:general} is that we no longer require the term $\Psi(\wmin,N)$ in \eqref{eqn:init_upper}.

\begin{remark}
\label{remark:init}
We investigate how the depth $N$ influences the requirement on initialization due to the change on gradient dynamics.
We rewrite $\Phi(\wmax,\wmin, \epsilon,N)$ in terms of $w_0 = \alpha^N$, and therefore the upper bound for $w_0$ under the simplified setting of non-negative signals (Theorem \ref{thm:non-negative}) is 
\[
w_0 \leq 
 \left(\frac18\right)^{2N/(N-2)} 
\wedge
\left(\frac{(\wmax)^{(N-2)/N}}{\log \frac{\wmax}{\epsilon}} \right)^{2N/(N-2)}
\wedge 
\left(\frac{(\wmin)^{(N-2)/N}}{ \log \frac{\wmin}{\epsilon}}\right)^{4N/(N-2)}.
\]
We start by analyzing each term in the upper bound. First, we notice that $\left(\frac18\right)^{2N/(N-2)}$ is increasing with respect to $N$. For the second term,
\[
\left(\frac{(\wmax)^{(N-2)/N}}{\log \frac{\wmax}{\epsilon}} \right)^{2N/(N-2)} = \frac{(\wmax)^{2}}{(\log \frac{\wmax}{\epsilon})^{2N/(N-2)}},
\]
the denominator gets smaller as $N$ increases when we pick the error tolerance parameter $\epsilon$ small. Therefore, we get that the second term is getting larger as $N$ increases. The last term $\left(\frac{(\wmin)^{(N-2)/N}}{ \log \frac{\wmin}{\epsilon}}\right)^{4N/(N-2)}$ follows a similar argument. We see that it is possible to pick a larger initialization $w_0=\alpha^N$ for larger $N$. We will demonstrate that below in our experiments. 
\end{remark}

\textbf{(b) Early stopping.} Early stopping is shown to be crucial, if not necessary, for implicit sparse regularization \cite{vaskevicius2019implicit,zhao2019implicit}. Interestingly, \cite{gissin2019implicit, woodworth2020kernelregimes} studied the similar depth-$N$ polynomial parametrization but did not realize the need of early stopping due to an oversimplification in the model. We will discuss this in details in Section \ref{sec:simplified}. We are able to explicitly characterize the window of the number of iterations that are sufficient to guarantee the optimal result. In particular, we get a lower bound of the window size for early stopping to get a sense of how it changes with different $N$.

\begin{theorem}[Informal]
Define the early stopping window size as $T_u(\wstar,\alpha,N,\eta,\zeta,\epsilon) - T_l(\wstar,\alpha,N,\eta,\zeta,\epsilon)$, the difference between the upper bound and lower bound of the number of iterations in \eqref{eqn:iteration_bound} of Theorem \ref{thm:general}.
Fixing $\alpha$ and $\eta$ for all $N$, the early stopping window size is increasing with $N$ under mild conditions.
\label{thm:informal-early-stopping}
\end{theorem}

We defer the formal argument and proof of Theorem \ref{thm:informal-early-stopping} to Appendix \ref{sec:early-stopping}.
We note that the window we obtain in Theorem \ref{thm:general} is not necessarily the largest window that allows the guarantee, and hence the early stopping window size can be effectively regarded a lower bound of that derived from the largest window.
We note that a precise characterization of the largest window is difficult.
Although we only show that this lower bound increases with $N$, we see that the conclusion matches empirically with the largest window.
Fix the same initialization $\alpha=0.005$ and step size $\eta=0.01$ for $N=2,3,4$, we show the coordinate path in Figure \ref{fig:window_size_cor}. We can see that as $N$ increases, the early stopping window increases and the error bound captures the time point that needs stopping quite accurately. 
The experimental details and more experiments about early stopping is presented in Section \ref{sec:sim}.

\begin{figure}[ht!]
    \centering
    \includegraphics[width=\linewidth]{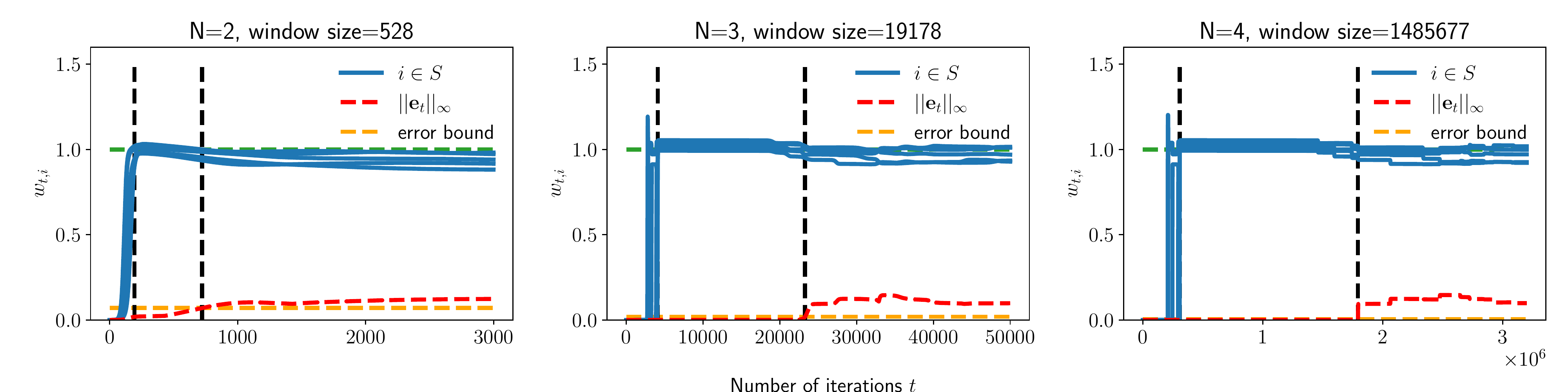}
    \caption{The black line indicates the early stopping window for different $N=2,3,4$. The blue line is the coordinate path for each entry on the support. The red line indicates the absolute value of the largest entry on the coordinate path outside the support. We use the orange line to indicate the requirement outside the support for early stopping.}
    \label{fig:window_size_cor}
\end{figure}

\revise{
\begin{remark}
\label{remark:init-early-stopping}
Similar to Theorem \ref{thm:informal-early-stopping}, we look at how initialization scale affects the early stopping window for any fixed $N>2$. With $\eta$ fixed, the early stopping window is increasing as the initialization $\alpha$ decreases. 
\end{remark}
We defer the detailed calculation to Section~\ref{sec:init-early-stopping}. This generalizes the finding that vanishing initialization increases the gap between the phase transition times in \cite{gidel2019implicit} from $N=2$ to any $N>2$.
}

\section{Proof Ingredients}
\label{sec:proofingred}
The goal of this paper is to understand how generalization and gradient dynamics change with different $N>2$. For $N=2$, gradient descent yields both statistically and computationally optimal recovery under the RIP assumption \cite{vaskevicius2019implicit}. The matrix formulation of the same type of parametrization is considered in the setting of low-rank matrix recovery, and exact recovery can be achieved in the noiseless setting \cite{gunasekar2018implicit, li2018algorithmic}. The key proof ingredient is to reduce the convergence analysis to one-dimensional iterates and differentiate the convergence on the support from the error outside the support. Before we get into that, we conduct a simplified gradient flow analysis.

\subsection{A Simplified Analysis}
\label{sec:simplified}

Consider a simplified problem where the target signal $\wstar$ is non-negative, $n^{-1} \XtX = \matrixid$ and  the noise is absent. We omit the reparametrization of $\vec{v}^N$ like before and the gradient descent updates on $\vec{u}$ will be independent for each coordinate. The gradient flow dynamics of $\vec{w} = \vec{u}^N$ is derived as 
\begin{equation}
    \frac{\partial w_i}{\partial t} 
    =
    \frac{\partial w_i}{\partial u_i} \frac{\partial u_i}{\partial t}
    =
    -\frac{\partial w_i}{\partial u_i} \frac{\partial \mathcal{L}}{\partial u_i}
    = 2N^2 (w^\star_i - w_i) w_i^{2-\frac2N}, 
    \label{eq:gf}
\end{equation}
for all $i \in \{1,2,\ldots,p\}$. Notice that $w_i$ increases monotonically and converges to $w_i^\star$ if $w_i^\star$ is positive or otherwise keeps decreasing and converges to $0$  if $w_i^\star =0$. As such, we can easily distinguish the support and non-support. In fact, gradient flow with dynamics as in \eqref{eq:gf} would exhibit a behavior of ``incremental learning'' --- the entries are learned separately, one at a time \cite{gissin2019implicit}. However, with the presence of noise and perturbation arising from correlated designs, the gradient flow may end up over-fitting the noise. Therefore, early stopping as well as the choice of step size is crucial for obtaining the desired solution \cite{vaskevicius2019implicit}. We use \eqref{eq:gf} to obtain a gradient descent update:
\begin{equation}
    w_{i,t+1} = w_{i,t}(1+2N^2\eta(w^\star_i-w_{i,t})w_{i,t}^{1-\frac2N}).
    \label{eq:ideal-gd}
\end{equation}

The gradient descent with $N=2$ is analyzed in \cite{vaskevicius2019implicit}. However, when $N>2$, the presence of $w_{i,t}^{1-\frac2N}$ imposes an asymmetrical effect on the gradient dynamics. The difficulty of analyzing such gradient descent \eqref{eq:ideal-gd} is pointed out in \cite{gissin2019implicit}.
More specifically, the recurrence relation is not solvable. However, gradient descent updates still share similar dynamics with the idealized gradient flow in \eqref{eq:gf}. Inspired by this effect, we are able to show that the entries inside the support and those outside the support are learned separately with a practical optimization algorithm shown in \eqref{eq:gd} and \eqref{eq:gd on u,v}. As a result, we are able to explore how the depth $N$ affects the choice of step size and early stopping criterion.

\subsection{Proof Sketch}

\textbf{Growth rate of gradient descent.}
We adopt the same decomposition as illustrated in \cite{vaskevicius2019implicit}, and define the following error sequences:

\begin{equation}
\vec{b}_{t} = \frac{1}{n}\XtX \vec{e}_{t} - \frac{1}{n}\Xt\vec{\xi}, 
\quad
\vec{p}_{t} = \left( \frac{1}{n}\XtX - \matrixid \right) \left(\vec{s}_{t} - \wstar \right),
\label{eq:b-p-decomposition}
\end{equation}

where $\vec{e}_t$ and $\vec{s}_t$ stand for error and signal accordingly, and the definitions can be found in \eqref{eq:error-decompositions} in Appendix.
We can then write the updates on $\vec{s}_t$ and $\vec{e}_t$ as 
\begin{equation}
\begin{aligned}
\vec{s}_{t+1} &= \vec{s}_t \odot (\id{} - 2N\eta(\vec{s}_t - \wstar + \vec{p}_t+\vec{b}_t)\odot \vec{s}_t^{(N-2)/N})^N,\\
\vec{e}_{t+1} &= \vec{e}_t \odot (\id{} - 2N\eta(\vec{p}_t+\vec{b}_t)\odot \vec{e}_t^{(N-2)/N})^N.
\end{aligned}
\end{equation}

 To illustrate the idea, we think of the one-dimensional updates $\{s_t\}_{t\geq 0}$ and $\{e_t\}_{t\geq 0}$, ignore the error perturbations $\vec{p}_t$ and $\vec{b}_t$ in the signal updates $\{s_t\}_{t\geq 0}$, and treat $\norm{\vec{p}_t+\vec{b}_t}_\infty \leq B$ in the error updates $\{e_t\}_{t\geq 0}$. 
\begin{equation}
    s_{t+1} = s_t(1-2N\eta(s_t-w^\star)s_t^{(N-2)/N})^N,
    \quad
    e_{t+1} = e_t(1-2N\eta B e_t^{(N-2)/N})^N.
    \label{eq:signal-error-updates}
\end{equation}
We use the continuous approximation to study the discrete updates. Therefore, we can borrow many insights from the analysis about gradient flow to overcome the difficulties caused by $w_{i,t}^{1-\frac2N}$ as pointed out in equation \eqref{eq:ideal-gd}.
With a proper choice of step size $\eta$, the number of iterations $T_l$ for $s_t$ converging to $w^\star$ is derived as
\begin{align*}
T_l &\leq \sum_{t=0}^{T_l-1} \frac{s_{t+1}-s_t}{2N^2\eta(w^\star-s_t) s_t^{(2N-2)/N}}
&\leq \frac{1}{N^2\eta w^\star} \int_{\alpha^N}^{w^\star} \frac{1}{s^{(2N-2)/N}} ds + \mathcal{O}\left(\frac{w^\star - \alpha^N}{\alpha^{2N-2}}\right).
\end{align*}
The number of iterations $T_u$ for $e_t$ staying below some threshold $\alpha^{N/4}$ is derived as 
\[
T_u \geq \sum_{t=0}^{T_u-1} \frac{e_{t+1}-e_t}{4N^2\eta B e_t^{(2N-2)/N}}
\geq \frac{1}{4N^2\eta B} \int_{\alpha^N}^{\alpha^{N/4}} \frac{1}{e^{(2N-2)/N}} de.
\]

With our choice of coherence $\mu$ in Theorem \ref{thm:general}, we are able to control $B$ to be small so that $T_l$ is smaller than $T_u$. This means the entries on the support converge to the true signal while the entries outside the support stay around 0, and we are able to distinguish signals and errors.

\textbf{Dealing with negative targets.}
We now illustrate the idea about how to generalize the result about non-negative signals to general signals. 
The exact gradient descent updates on $\vec{u}$ and $\vec{v}$ are given by:
\begin{equation}
    \begin{aligned}
\vec{u}_{t+1} &= \vec{u}_t \odot \left(\vec{1} - 2N\eta \left(\frac1n \Xt(\X(\vec{w}_t-\vec{w}^\star)-\vec{\xi})\odot \vec{u}_t^{N-2}\right)\right),\\
\vec{v}_{t+1} &= \vec{v}_t \odot \left(\vec{1} + 2N\eta \left(\frac1n \Xt(\X(\vec{w}_t-\vec{w}^\star)-\vec{\xi})\odot \vec{v}_t^{N-2}\right)\right).
\end{aligned}
\label{eq:gd on u,v}
\end{equation}

The basic idea is to show that when $w_i^\star$ is positive, $v_i^\star$ remains small up to the early stopping criterion, and when $w_i^\star$ is negative, $u_i^\star$ remains small up to the early stopping criterion. We turn to 
studying the gradient flow of such dynamics. Write $\vec{r}(t)=\frac1n \Xt(\X(\vec{w}(t)-\vec{w}^\star)-\vec{\xi}$. It is easy to verify that the gradient flow has a solution:
\begin{align*}
    \vec{u}(t) = \left(\alpha^{2-N}\id{} + 2N(N-2)\eta \int_0^t \vec{r}(\upsilon) 
    d\upsilon\right)^\frac{1}{2-N},\\
    \vec{v}(t) = \left(\alpha^{2-N}\id{} - 2N(N-2)\eta \int_0^t \vec{r}(\upsilon) 
    d\upsilon\right)^\frac{1}{2-N}.
\end{align*}
 We may observe some symmetry here, when $u_{i,t}$ is large, $v_{i,t}$ must be small. 
 For the case $w_i>0$, to ensure the increasing of $u_{i,t}$ and decreasing of $v_{i,t}$ as we desire, 
 the initialization needs to be smaller than $w_i$, which leads to the extra constraint on initialization $\Psi(\wmin,\epsilon)$ with order of $\mathcal{O}(\wmin)$ as defined before.
 It remains to build the connection between gradient flow and gradient descent, where again we uses the continuous approximation as before. 
 The detailed derivation is presented in Appendix \ref{sec: neg}.
 
\section{Simulation Study}
\label{sec:sim}
We conduct a series of simulation experiments\footnote{ \revise{The code is available on \url{https://github.com/jiangyuan2li/Implicit-Sparse-Regularization}.}} to further illuminate our theoretical findings.
Our simulation setup is described as follows. The entries of $\X$ are sampled as i.i.d. Rademacher random variables and the entries of the noise vector $\vec{\xi}$ are i.i.d. $N(0,\sigma^2)$ random variables. We let $\wstar = \gamma\id{S}$.  
The values for the simulation parameters are: $n=500$, $p=3000$, $k=5$, $\gamma=1$, $\sigma=0.5$ unless otherwise specified.  For $\ell_2$-plots each simulation is repeated 30 times, and the median $\ell_2$ error is depicted. The shaded area indicates the region between $25^\textrm{th}$ and $75^\textrm{th}$ percentiles pointwisely.

\textbf{Convergence results.} We start by showing that the general choice of $N$ leads to the sparse recovery, similar to $N=2$ in \cite{vaskevicius2019implicit}, as shown in our main theorem. We choose different values of $N$ to illustrate the convergence of the algorithm.
\revise{
The result on simulated data is shown in Figure~\ref{fig:convergence}, and we defer the result on MNIST to Appendix~\ref{sec:more-experiments}.}
Note that the ranges in the $x$-axes of these figures differ due to different choice of $N$ and $\eta$. We observe that as $N$ increases, the number of iterations increases significantly. This is due to the term $\vec{u}^{N-2}$ and $\vec{v}^{N-2}$ in \eqref{eq:gd on u,v}, and the step size $\eta \approx \frac{1}{N^2}$.
With a very small initialization, it takes a large number of iterations to escape from the small region (close to $0$).
\begin{figure}[!ht]
    \centering
    \includegraphics[width=\linewidth]{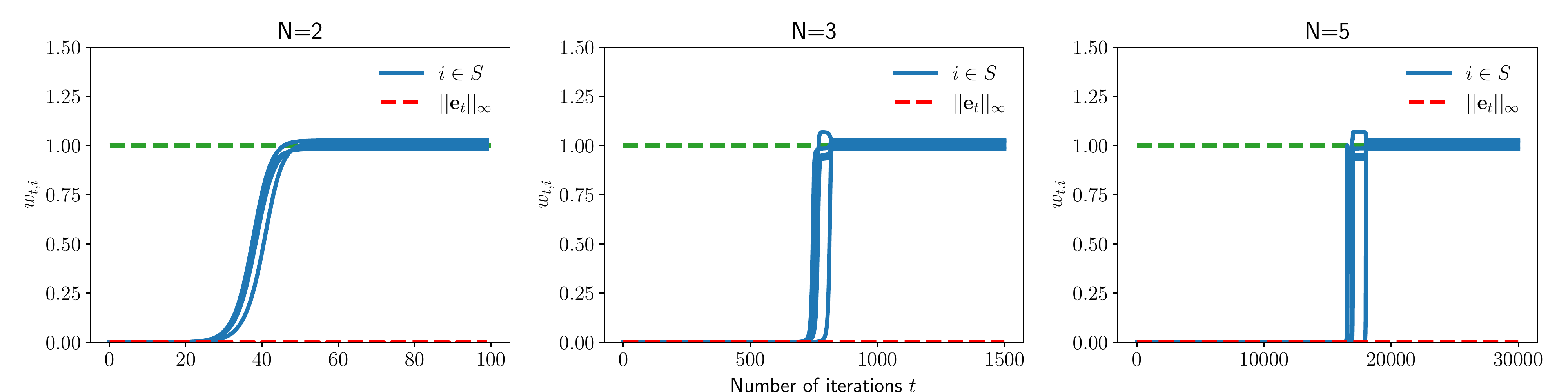}
    \caption{Coordinates paths for different choice of $N=2,3,5$ with $\alpha^N=10^{-6}$ and $\eta=1/{(5N^2)}$.}
    \label{fig:convergence}
\end{figure}

\textbf{Larger initialization.} As discussed in Remark \ref{remark:init}, the upper bound on initialization gets larger with larger $N$. We intentionally pick a relatively large $\alpha^N=2\times 10^{-3}$ 
where the algorithm fails to converge for $N=2$. With the same initialization, the recovery manifests as $N$ increases (Figure\ \ref{fig:large-init}).
\begin{figure}[ht!]
    \centering
    \includegraphics[width=\linewidth]{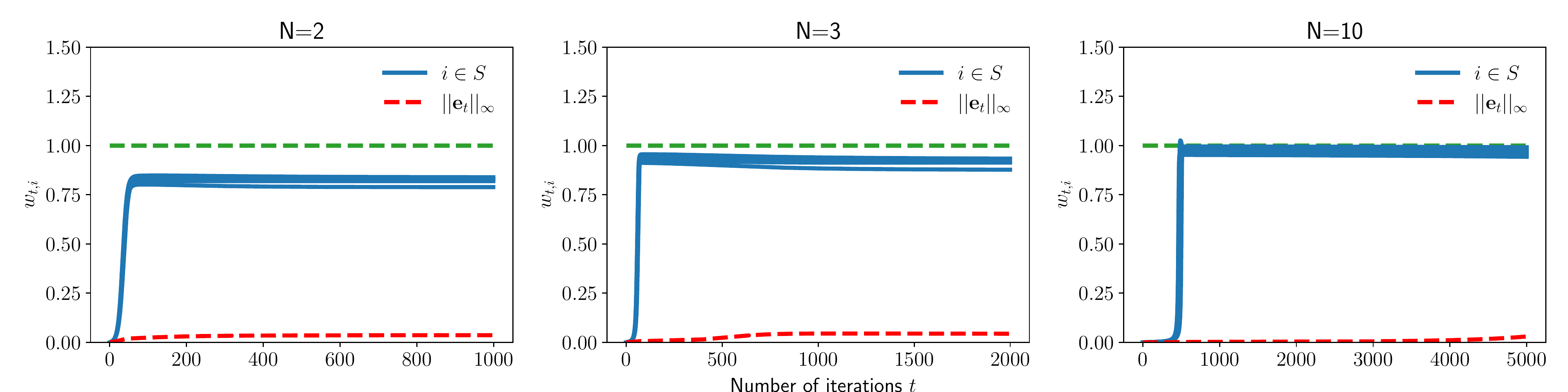}
    \caption{The effect of $N$ on the initialization $\alpha^N$ with $\eta=1/{(5N^2)}$.}
    \label{fig:large-init}
\end{figure}

\textbf{Early stopping window size.} Apart from the coordinate path shown in Figure \ref{fig:window_size_cor}, we obtain multiple runs and plot the $\log$-$\ell_2$ error (the logarithm of the $\ell_2$-error) of the recovered signals to further confirm the increase of early stopping window, as shown in Section \ref{sec:main-res}. 
Note that for both Figures \ref{fig:window_size_cor} and \ref{fig:early_stopping_error}, we set $n=100$ and $p=200$. 
Since $\alpha^N$ would decrease quickly with $N$, which would cause the algorithm takes a large number of iterations to escape from the small region. 
We fix $\alpha^N=10^{-5}$ instead of fixing $\alpha$ for Figure \ref{fig:early_stopping_error}.
\begin{figure}[ht!]
    \centering
    \includegraphics[width=\linewidth]{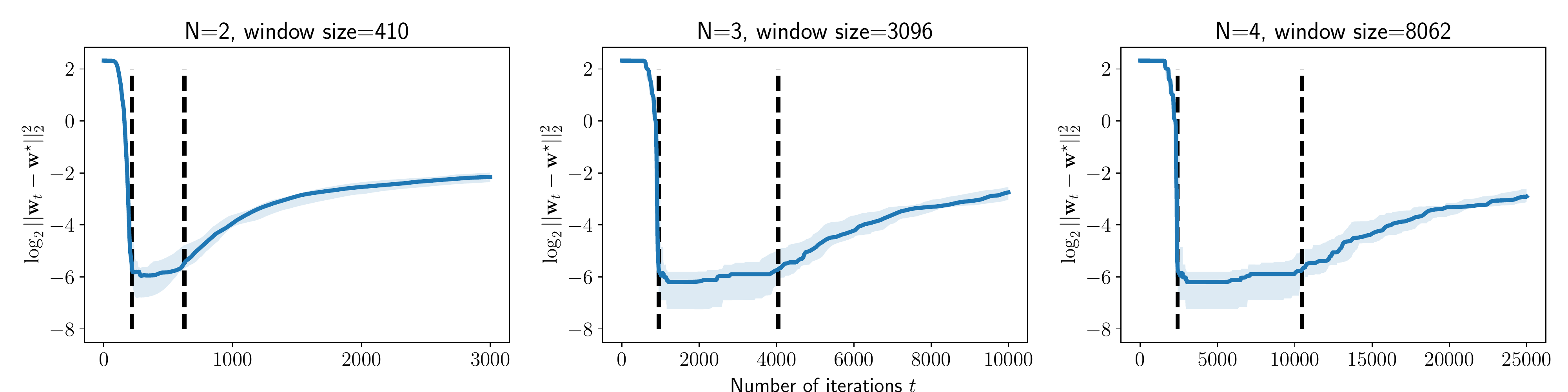}
    \caption{$\log$-$\ell_2$ error of $N=2,3,4$ with the fixed step size $\eta=0.01$.}
    \label{fig:early_stopping_error}
\end{figure}

\textbf{Incremental learning dynamics.} The dynamics of incremental learning for different $N$ is discussed in \cite{gissin2019implicit}. The distinct phases of learning are also observed in sparse recovery (Figure \ref{fig:incremental_learning}), though we do not provide a theoretical justification. Larger values of $N$ would lead to more distinct learning phases for entries with different magnitudes under the same initialization $\alpha^N$ and step size $\eta$.

\begin{figure}[ht!]
    \centering
    \includegraphics[width=\linewidth]{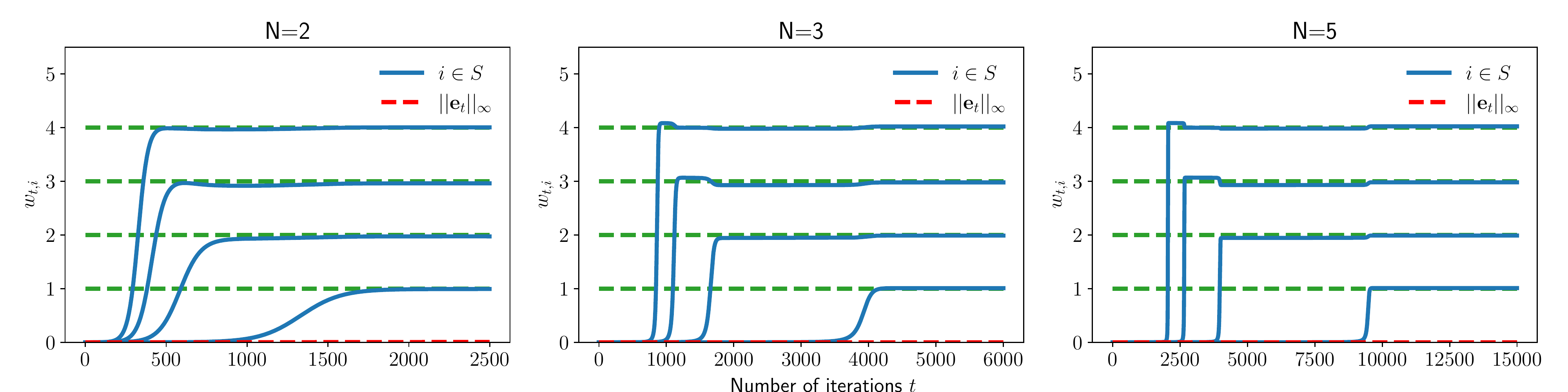}
    \caption{Coordinates paths for $N=2,3,5$. The entries of $\wstar$ on the support $S$ are now $[1,2,3,4]$. The initialization is $\alpha^N=10^{-4}$ and the step size is $\eta = 10^{-3}$ for all $N$.}
    \label{fig:incremental_learning}
\end{figure}

\textbf{Kernel regime.} 
As pointed out in \cite{woodworth2020kernelregimes}, the scale of initialization determines whether the gradient dynamics obey the ``kernel'' or ``rich'' regimes for diagonal linear networks. 
We have carefully analyzed and demonstrated the sparse recovery problem with small initialization, which corresponds to the ``rich'' regime. 
To explore the "kernel" regime in a more practical setting, we set $n=500$, $p=100$, and the entries of $\wstar$ are i.i.d. $\mathcal{N}(0,1)$ random variables. The noise level is $\sigma=25$, and the initialization and step size is set as $\alpha^N=1000$ and $\eta=10^{-7}$ for all $N$.
Note that we are not working in the case $n\ll p$ as \cite{woodworth2020kernelregimes}. We still observe that the gradient dynamics with large initialization (Figure \ref{fig:ridge_regime}) can be connected to ridge regression if early stopping is deployed.

\begin{figure}[ht!]
    \centering
    \includegraphics[width=\linewidth]{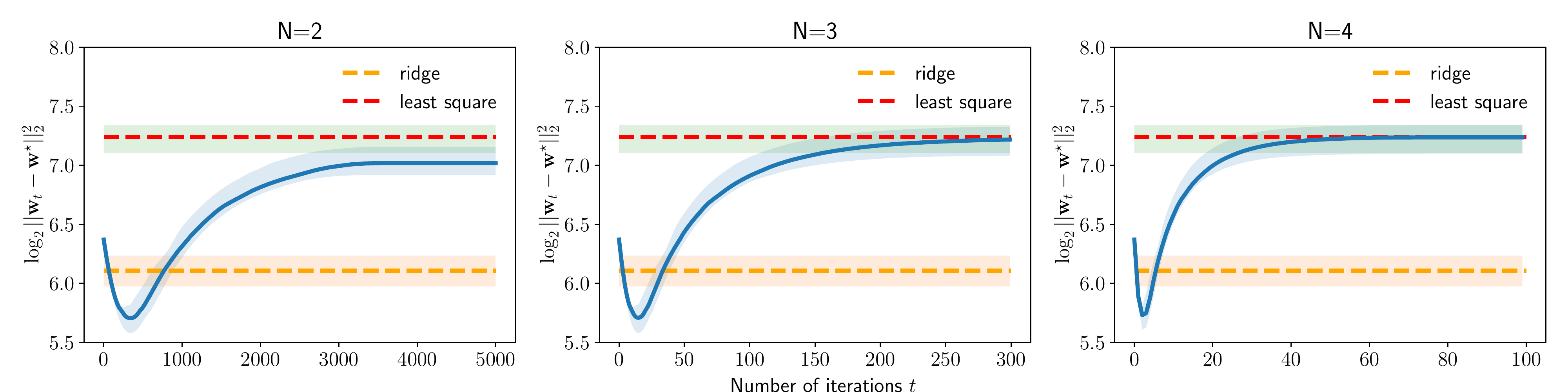}
    \caption{$\log$-$\ell_2$ error of $N=2,3,4$ for a ridge regression setting. The ridge regression solution is selected by 5-fold cross validation.}
    \label{fig:ridge_regime}
\end{figure}

\section{Conclusions and Future Work}
In this paper, we extend the implicit regularization results in \cite{vaskevicius2019implicit} from $N=2$ to general $N>2$, and further study how gradient dynamics and early stopping is affected by different choice $N$. 
We show that the error bound is invariant with different choice of $N$ and yields the minimax optimal rate. The step size is of order $\mathcal{O}(1/N^2)$. The initialization and early stopping window gets larger when increasing $N$ due to the changes on gradient dynamics.
\revise{
The incremental learning dynamics and kernel regime of such parametrizations are empirically shown, however not theoretically justified, which is left for future work.
}

The convergence result can be further improved by relaxing the requirement on the incoherence of design matrix from $\mu\lesssim\frac{1}{k\wmax/\wmin}$ to $\mu\lesssim\frac{1}{k\log(\wmax/\wmin)}$, similar to \cite{vaskevicius2019implicit}. Overall, we believe that such an analysis and associated techniques could be applied for studying other, deeper nonlinear models in more practical settings.

\revise{
\section*{Acknowledgements}
This work was supported in part by the National Science Foundation under grants CCF-1934904, CCF-1815101, CCF-2005804, and DMS-1711952.
}
\bibliographystyle{unsrt}

{
\small
\bibliography{arxiv}}



\newpage

\appendix
\section*{Appendix}
The appendix is organized as follows.

In Appendix~\ref{sec:proof-for-non-negative}, we present a simplied theorem about non-negative signals and illustrate the idea behind the proof.

In Appendix~\ref{sec:multiplicative-updates}, we study the multiplicative updates and build connections to its continuous approximation, which will be used next.

In Appendix~\ref{sec: proof of prop}, we provide the proof of propositions and technical lemmas in Appendix~\ref{sec:proof-for-non-negative}.

In Appendix~\ref{sec:proof-of-main-res}, we prove the main results stated in the paper.

In Appendix~\ref{sec:more-experiments}, we provide the experimental results on real-world datasets to illustrate the effectiveness of the proposed algorithm.

\section{Proof for Non-negative Signals}
\label{sec:proof-for-non-negative}
We mainly follow the proof structure from \cite{vaskevicius2019implicit} to obtain the convergence of similar gradient descent algorithm for the case $N=2$, which is a limiting case of ours. We will demonstrate how gradient dynamics changes with $N>2$, which requires us to study the growth rate of error and convergence rate more carefully.

In this section, we will start with the general set up and provide a simplified version of Theorem \ref{thm:general} about non-negative signals. 

\subsection{Setup}
The gradients of $\mathcal{L}(\vec{u},\vec{v})$ with respect to $\vec{u}, \vec{v}$ read as
\[
\begin{aligned}
\nabla_{\vec{u}} \mathcal{L}(\vec{w}) &= \frac{2N}{n} \Xt(\X\vec{w}-\vec{y})\odot \vec{u}^{N-1}\\
\nabla_{\vec{v}} \mathcal{L}(\vec{w}) &= -\frac{2N}{n} \Xt(\X\vec{w}-\vec{y})\odot \vec{v}^{N-1}.
\end{aligned}
\]
With the step size $\eta$, the gradient descent updates on $\vec{u}_t$ and $\vec{v}_t$ simply are
\[
\begin{aligned}
\vec{u}_{t+1} &= \vec{u}_t \odot \left(\vec{1} - 2N\eta \left(\frac1n \Xt(\X(\vec{w}_t-\vec{w}^\star)-\xxi)\odot \vec{u}_t^{N-2}\right)\right),\\
\vec{v}_{t+1} &= \vec{v}_t \odot \left(\vec{1} + 2N\eta \left(\frac1n \Xt(\X(\vec{w}_t-\vec{w}^\star)-\xxi)\odot \vec{v}_t^{N-2}\right)\right).
\end{aligned}
\]

Let $\vec{w}_t = \vec{w}_t^+ - \vec{w}_t^-$ where $\vec{w}_t^+ \coloneqq \vec{u}_t^N$ and $\vec{w}_t^- \coloneqq \vec{v}_t^N$ with the power taken element-wisely. 
We denote $S$ as the support of $\wstar$, and let $S^{+} = \{i | w^{\star}_{i} > 0 \}$ denote the index set of coordinates with positive values, and $S^{-} = \{ i | w^{\star}_{i} < 0\}$ denote the index set of coordinates with negative values. Therefore $S=S^+\cup S^-$ and $S^+\cap S^- =\emptyset$.
Then define the following signal and noise-related quantities:
\begin{align}
\begin{split}
\label{eq:error-decompositions}
\vec{s}_{t}
&\coloneqq\id{S^{+}} \odot \vec{w}_{t}^{+} - \id{S^{-}} \odot \vec{w}_{t}^{-}, \\
\vec{e}_{t} 
&\coloneqq \id{S^{c}} \odot \vec{w}_{t} + \id{S^{-}} \odot \vec{w}_{t}^{+} - \id{S^{+}} \odot \vec{w}_{t}^{-}, \\
\vec{b}_{t}
&\coloneqq \frac{1}{n}\XtX \vec{e}_{t} - \frac{1}{n}\Xt\vec{\xxi}, \\
\vec{p}_{t}
&\coloneqq \left( \frac{1}{n}\XtX - \matrixid \right) \left(\vec{s}_{t} - \wstar \right).
\end{split}
\end{align}

Let $\alpha^{N}$ be the initial value for each entry of $\vec{w}$ and rewrite the updates on
$\vec{w}_{t}$, $\vec{w}_{t}^{+}$ and $\vec{w}_{t}^{-}$
in a more succinct way:
\begin{align}
\begin{split}
\label{eq:updates-equation-using-b-p-notation}
&\vec{w}_{0}^{+} = \vec{w}_{0}^{-} = \alpha^{N}, \\
&\vec{w}_{t} = \vec{w}_{t}^{+} - \vec{w}_{t}^{-}, \\
&\vec{w}_{t+1}^{+} = \vec{w}_{t}^{+} \odot \left( \id{} - 2N\eta 
\left( \vec{s}_{t} - \wstar + \vec{p}_{t} + \vec{b}_{t} \right)\odot(\vec{w}_t^+)^{(N-2)/N} \right)^{N}, \\
&\vec{w}_{t+1}^{-} = \vec{w}_{t}^{-} \odot 
\left( \id{} + 2N\eta \left( \vec{s}_{t} - \wstar + \vec{p}_{t} + \vec{b}_{t} \right)\odot(\vec{w}_t^-)^{(N-2)/N} \right)^{N}.
\end{split}
\end{align}

When our target $\wstar$ is with non-negative entries, the design of $\vec{v}_t$ is no longer needed and the algorithm could be simplied to the following form.
\begin{equation}
\begin{aligned}
\vec{w}^+_0 
&= \vec{u}_0^N = \alpha^N,\\
\vec{w}^+_t 
&= \vec{u}_t^N,\\
\vec{w}_{t+1}^{+}
&= \vec{w}_{t}^{+} \odot
\left(\id{} - 2N\eta \left(\vec{s}_{t} - \wstar + \vec{p}_{t} + \vec{b}_{t} \right)\odot(\vec{w}_t^+)^{(N-2)/N}\right)^{N}
\end{aligned}
\label{eq:updates-for-non-negative}
\end{equation}
The results in this section are all about updates in equation 
\eqref{eq:updates-for-non-negative}, and will be generalized to updates in equation \eqref{eq:updates-equation-using-b-p-notation} in Section \ref{sec:proof-of-main-res}.

\subsection{The Key Propositions}
\label{sec:prop}
Starting from $t=0$, we have $\norm{\vec{s}_0-\wstar}_\infty \lesssim \mathcal{O}(\wmax)$ and $\norm{\vec{e}_0}_\infty \leq \alpha^N$. The idea of proposition \ref{prop: first stage} is to show that after some certain number of iterations $t$, we obtain $\norm{\vec{s}_t-\wstar}_\infty \lesssim \mathcal{O}(\wmin)$ and $\norm{\vec{e}_t}_\infty \leq \alpha^{N/2}$. Proposition \ref{prop: second stage} further reduces the approximation error from $\mathcal{O}(\wmin)$ to $\mathcal{O}(\norm{\frac{1}{n} \Xt \vec{\xxi}}_\infty)$ if possible, while still maintaining $\norm{\vec{e}_t}_\infty \leq \alpha^{N/4}$.

\begin{proposition}
Consider the updates in equations \eqref{eq:updates-for-non-negative}. Fix any $0<\zeta\leq\wmax$ and let $\gamma=C_\gamma\frac{\wmin}{\wmax}$ where $C_\gamma$ is some small enough absolute constant. Suppose the error sequences $(\vec{b}_t)_{t\geq0}$ and $(\vec{p}_t)_{t\geq0}$ for any $t\geq0$ satisfy the following:

\begin{align*}
\norm{\vec{b}_t}_\infty &\leq C_b \zeta -\alpha^{N/4},\\
\norm{\vec{p}_t}_\infty &\leq \gamma\norm{\vec{s}_t-\wstar}_\infty,
\end{align*}

where $C_b$ is some small enough absolute constants. If the initialization satisfies 
\[
\alpha 
\leq 
\left(\frac18\right)^{2/(N-2)}
\wedge
\left( \frac{(\wmax)^{(N-2)/N}}{\log \frac\wmax\epsilon}\right)^{2/(N-2)},
\]
and the step size $\eta\leq \frac{\alpha^N}{8N^2\zeta^{(3N-2)/N}}$, then for any $T_1 \leq T\leq T_2$ where
\begin{align*}
T_1 &= \frac{75}{16\eta N^2 \zeta^{(2N-2)/N}} \log \frac{|\wmax - \alpha^N|}{\epsilon}
    +
    \frac{15}{8N(N-2)\eta \zeta \alpha^{(N-2)}},\\
T_2&=\frac{5}{N(N-1)\eta\zeta}\left(\frac{1}{\alpha^{(N-2)}}-\frac{1}{\alpha^{(N-2)/2}}\right),\\
\end{align*}

and any $0\leq t\leq T$, we have 
\begin{align*}
\norm{\vec{s}_T-\wstar}_\infty &\leq \zeta,\\
\norm{\vec{e}_t}_\infty &\leq \alpha^{N/2}.
\end{align*}
\label{prop: first stage}
\end{proposition}

Note that the requirement on $\norm{\vec{b}_t}_\infty \leq C_b\zeta-\alpha^{N/4}$ can be relaxed to $\norm{\vec{b}_t}_\infty \leq C_b\zeta$ when we just consider the updates in equation \eqref{eq:updates-for-non-negative}. However, we still consider the stronger requirement in order to further generalize to updates in equation \eqref{eq:updates-equation-using-b-p-notation} later.

\begin{proposition}
Consider the updates in equations \eqref{eq:updates-for-non-negative}. Fix any $0<\zeta\leq\wmax$ and suppose that the error sequences $(\vec{b}_t)_{t\geq0}$ and $(\vec{p}_t)_{t\geq0}$ for any $t\geq 0$ satisfy
\begin{align*}
    B &= \norm{\vec{b}_t}_\infty + \norm{\vec{p}_t}_\infty \leq \frac{1}{200}\wmin\\
    \norm{\vec{b}_t\odot \id{i}}_\infty
    &\leq B_i \leq \frac{1}{10}\wmin,\\
    \norm{\vec{p}_t}_\infty 
    &\leq\frac{1}{20}\norm{\vec{s}_0 - \wstar}_\infty.
\end{align*}
Suppose that
\begin{align*}
\alpha 
\leq \left(\frac14\right)^{2/(N-2)} 
&\wedge
\left(\frac{(\wmin)^{(N-2)/N}}{\log \frac{\wmin}{\epsilon}}\right)^{4/(N-2)},\\
\norm{\vec{s}_0 - \wstar}_\infty 
&\leq \frac15 \wmin,\\
\norm{\vec{e}_0} 
&\leq \alpha^{N/2}.
\end{align*}
Let the step size satisfy $\eta\leq \frac{\alpha^N}{8N^2(\wmin)^{(3N-2)/N}}$. Then for any $T_3\leq t\leq T_4$,
\begin{align*}
    T_3 &= \frac{6}{\eta N^2 (\wmin)^{(2N-2)/N}}\log \frac{\wmin}{\epsilon},\\
    T_4 &= \frac{25}{N(N-1)\eta \wmin}\left(\frac{1}{\alpha^{(N-2)/2}} -
    \frac{1}{\alpha^{(N-2)/4}}\right),\\
\end{align*}
and any $i\in S$ we have
\begin{align*}
|s_{i,t}-w_i^\star| &\lesssim k\mu\max_{j\in S}B_j\vee B_i\vee\epsilon,\\
\norm{\vec{e}_t}_\infty &\leq \alpha^{N/4}.
\end{align*}
\label{prop: second stage}
\end{proposition}

\subsection{Technical Lemmas}
\label{sec: tech lemma}
There are several lemmas, which are about the coherence of the design matrices and the upper bound of subGaussian noise term.

\begin{lemma}
  \label{lemma:max-eroor}
  Suppose that $\frac{1}{\sqrt{n}}\X$ is a $n \times p$ matrix with $\ell_2$-normalized columns and satisfies $\mu$-coherence with $0 \leq \mu \leq 1$.
  Then for any vector $\vec{z} \in \mathbb{R}^{p}$ we have
  $$
    \norm{\frac{1}{n} \XtX \vec{z}}_{\infty} \leq p\norm{\vec{z}}_{\infty}.
  $$
\end{lemma}

\begin{lemma}
  \label{lemma:incoherence-assumption-pt}
  Suppose that $\frac{1}{\sqrt{n}}\X$ is a $n \times p$ $\ell_2$-normalized matrix satisfying
  $\mu$-incoherence; that is $\frac{1}{n}|\vec{X}_i^\top \vec{X}_j| \le \mu, i\ne j$.
  For $k$-sparse vector $\vec{z} \in \mathbb{R}^{p}$, we have:
  \[
    \norm{\left( \frac{1}{n} \XtX - \matrix{I} \right)\vec{z}}_{\infty}
    \leq
    k \mu \norm{\vec{z}}_{\infty}.
  \]
\end{lemma}

\begin{lemma}
  \label{lemma:bounding-max-noise}
  Let $\frac{1}{\sqrt{n}} \X$ be a $n \times p$ matrix with $\ell_2$-normalized columns. Let $\xxi \in \mathbb{R}^{n}$ be a vector of independent
  $\sigma^{2}$-sub-Gaussian random variables. Then, with probability at
  least $1 - \frac{1}{8p^{3}}$
  $$
  \norm{\frac1n \Xt \xxi}_\infty \lesssim \sqrt{\frac{\sigma^{2} \log p}{n}}.
  $$
\end{lemma}

\subsection{Proof for Non-negative Signals}
Recall the notation
\[
\Phi(\wmax,\wmin,\epsilon,N) \coloneqq
\left(\frac18\right)^{2/(N-2)} 
\wedge
\left(\frac{(\wmax)^{(N-2)/N}}{\log \frac{\wmax}{\epsilon}} \right)^{2/(N-2)}
\wedge 
\left(\frac{(\wmin)^{(N-2)/N}}{ \log \frac{\wmin}{\epsilon}}\right)^{4/(N-2)},
\]
and 
\[
\zeta 
\coloneqq 
\frac15 \wmin 
\vee 
\frac{200}{n}\norm{\Xt\vec{\xxi}}_\infty 
\vee 
200\epsilon.
\]

\begin{theorem}
Suppose that $\wstar \succcurlyeq 0$ with $k\geq 1$ and $\X/\sqrt{n}$ satisfies $\mu$-incoherence with $\mu\leq C_\gamma/k r$, where $C_\gamma$ is some small enough constant. Take any precision $\epsilon >0$, and let the initialization be such that
\[
0<\alpha \leq \left(\frac{\epsilon}{p+1}\right)^{4/N}
\wedge \Phi(\wmax,\wmin,\epsilon,N) 
\]
For any iteration $t$ that satisfies
\[
\frac{1}{\eta N^2\zeta^{(2N-2)/N} \alpha^{N-2}}
\lesssim t \lesssim
\frac{1}{\eta N^2 \tau}
\left(\frac{1}{\alpha^{N-2}} 
-
\frac{1}{\zeta^{(N-2)/2}}\right),
\]
the gradient descent algorithm \eqref{eq:updates-for-non-negative} with step size $\eta\leq \frac{\alpha^N}{8N^2(\wmax)^{(3N-2)/N}}$ yields the iterate $\vec{w}_t$ with the following property:
\begin{equation}
|w_{t,i}-w_i^\star|\lesssim 
\begin{cases}
    \norm{\frac1n \Xt\vec{\xxi}}_\infty \vee \epsilon \quad &\text{if }i\in S \text{ and }\wmin\lesssim\norm{\frac1n \Xt\vec{\xxi}}_\infty \vee \epsilon, \\
    \left|\frac1n (\Xt\vec{\xxi})_i\right|
    \vee
    k\mu\norm{\frac1n \Xt\vec{\xxi}\odot \id{S}}_\infty 
    \vee
    \epsilon
    \quad &\text{if }i\in S \text{ and }\wmin \gtrsim\norm{\frac1n \Xt\vec{\xxi}}_\infty \vee \epsilon, \\
    
    \alpha^{N/4}\quad &\text{if }i\notin S.
\end{cases}
\label{eq:non-negative-error-bound}
\end{equation}
\label{thm:non-negative}
\end{theorem}

\begin{proof}
Let 
\[
\zeta \coloneqq \frac15 \wmin \vee 
\frac{2}{C_b} \norm{\frac1n \Xt \xxi}_\infty \vee
\frac{2}{C_b}\epsilon,
\]
where $C_b$ is some small enough positive constant that will be explicitly derived later. Also by the requirement of the coherence of the design matrix, we have
\[
\norm{\vec{p}_t}_\infty \leq 
\frac{C_\gamma}{\wmax/\wmin} \norm{\vec{s}_t-\wstar}_\infty.
\]

Setting 
\[
\alpha \leq  \left(\frac{\epsilon}{p+1}\right)^{4/N}
\wedge \left(\frac18\right)^{2/(N-2)} 
\wedge
\left(\frac{(\wmax)^{(N-2)/N}}{\log \frac{\wmax}{\epsilon}} \right)^{2/(N-2)}
\wedge 
\left(\frac{(\wmin)^{(N-2)/N}}{ \log \frac{\wmin}{\epsilon}}\right)^{4/(N-2)}.
\]

As long as $\norm{\vec{e}_t}_\infty \leq \alpha^{N/4}$ we have 
\begin{align*}
    \norm{\vec{b}_t}_\infty + \alpha^{N/4}
    &\leq \norm{\frac1n \Xt \epsilon}_\infty + \norm{\frac1n \XtX \vec{e}_t}_\infty + \alpha^{N/4} \\
    &\leq 2 \left(\norm{\frac1n \Xt \epsilon}_\infty \vee (p\norm{\vec{e}_t}_\infty\right) + \alpha^{N/4})\\
    &\leq 2 \left(\norm{\frac1n \Xt \epsilon}_\infty \vee 
    (p+1)\alpha^{N/4}\right)\\
    &\leq C_b \frac{2}{C_b} \left(\norm{\frac1n \Xt \epsilon}_\infty \vee \epsilon\right) \\
    &\leq C_b \zeta.
\end{align*}
where the second inequality is from Lemma \ref{lemma:max-eroor}.
Further by Lemma \ref{lemma:incoherence-assumption-pt}, we also have
\[
\norm{\vec{p}_t}_\infty \leq 
\frac{C_\gamma}{\wmax/\wmin} \norm{\vec{s}_t-\wstar}_\infty.
\]

Therefore, both sequences $(\vec{b}_t)_{t\geq0}$ and $(\vec{p}_t)_{t\geq0}$ satisfy the assumptions of Proposition \ref{prop: first stage} conditionally on $\norm{\vec{e}_t}_\infty$ staying below $\alpha^{N/4}$. If $\zeta \geq \wmax$, at $t=0$, we have already have 
\[
\norm{\vec{s}_0-\wstar}\leq \zeta.
\]
Otherwise, applying Proposition \ref{prop: first stage}, after
\[
T_1=\frac{75}{16\eta N^2 \zeta^{(2N-2)/N}} \log \frac{|\wmax - \alpha^N|}{\epsilon}
    +
    \frac{15}{8N(N-2)\eta \zeta \alpha^{(N-2)}},
\]
iterations and before
\[
T_2 = \frac{5}{N(N-1)\eta\zeta}\left(\frac{1}{\alpha^{(N-2)}}-\frac{1}{\alpha^{(N-2)/2}}\right)
\]
iterations, we have
\begin{align*}
\norm{\vec{s}_{T_1}-\wstar}&\leq \zeta,\\
\norm{\vec{e}_{T_1}}_\infty&\leq \alpha^{N/2}.
\end{align*}

If $\frac15 \wmin \leq \frac{2}{C_b} \norm{\frac1n \Xt \xxi}_\infty \vee \frac{2}{C_b}\epsilon$, then we are done.

If $\frac15 \wmin > \frac{2}{C_b} \norm{\frac1n \Xt \xxi}_\infty \vee \frac{2}{C_b}\epsilon$, we have $\zeta = \frac15 \wmin$. Choose $C_b+C_\gamma \leq \frac1{40}$ as we have in Proposition \ref{prop: first stage}. After $T_1$ iterations, we have
\[
\norm{\vec{b}_t}_\infty + \norm{\vec{p}_t}_\infty \leq C_b \frac15 \wmin + \frac{C_\gamma}{\wmax/\wmin} \frac15 \wmin \leq  (C_b+C_\gamma)\frac15 \wmin \leq \frac1{200} \wmin.
\]

Now all the assumptions of Proposition \ref{prop: second stage} are satisfied. To further reduce $\norm{\vec{s}_t-\wstar}_\infty$ from $\frac15 \wmin$ to $\mathcal{O}(\norm{\frac1n \Xt\xxi})$, we apply Proposition \ref{prop: second stage} and obtain that after 
\[
T_3 = \frac{6}{\eta N^2 (\wmin)^{(2N-2)/N}}\log \frac{\wmin}{\epsilon}
\]
iterations and before
\[
T_4 = \frac{25}{N(N-1)\eta \wmin}\left(\frac{1}{\alpha^{(N-2)/2}} -
    \frac{1}{\alpha^{(N-2)/4}}\right)
\]
iterations, we have for any $i\in S$,
\begin{align*}
|s_{t,i}-w_i^\star| &\lesssim k\mu\max_{j\in S}B_j\vee B_i\vee\epsilon,\\
\norm{\vec{e}_t}_\infty &\leq \alpha^{N/4}.
\end{align*}

We use $\mathbbm{1}{\{\cdot\}}$ to denote the indicator function. 
Therefore, the total number of iterations needed is 

\begin{equation}
\begin{aligned}
T_1 + T_3
&=\frac{75}{16\eta N^2 \zeta^{(2N-2)/N}} \log \frac{|\wmax - \alpha^N|}{\epsilon}
+
\frac{15}{8N(N-2)\eta \zeta \alpha^{(N-2)}}\\
&+
\frac{6}{\eta N^2 (\wmin)^{(2N-2)/N}}\log \frac{\wmin}{\epsilon} 
\mathbbm{1}{\left\{
\frac15 \wmin 
> 
\frac{2}{C_b}\norm{\frac1n \Xt \vec{\xxi}}_\infty 
\vee
\frac{2}{C_b}\epsilon
\right\}}
\end{aligned}
\label{eq:T1+T3}
\end{equation}

and the upper bound for the total number of iterations would be

\begin{equation}
\begin{aligned}
T_2 + T_4
&=\frac{5}{N(N-1)\eta\zeta}\left(\frac{1}{\alpha^{(N-2)}}-\frac{1}{\alpha^{(N-2)/2}}\right)\\
&+
\frac{25}{N(N-1)\eta \wmin}
\left(\frac{1}{\alpha^{(N-2)/2}} -
\frac{1}{\alpha^{(N-2)/4}}\right)
\mathbbm{1}{\left\{
\frac15 \wmin 
> 
\frac{2}{C_b}\norm{\frac1n \Xt \vec{\xxi}}_\infty 
\vee
\frac{2}{C_b}\epsilon
\right\}}
\end{aligned}
\label{eq:T2+T4}
\end{equation}
\end{proof}

\section{Multiplicative Update Sequences with General Order \texorpdfstring{$N$}{Lg}}
\label{sec:multiplicative-updates}
In this section, we analyze the one-dimensional updates that exhibits the similar dynamics to our gradient descent algorithm. The lemmas we derive will be assembled together to prove Proposition \ref{prop: first stage} and \ref{prop: second stage}. The whole framework is similar to \cite{vaskevicius2019implicit}. However, the continuous approximation plays an important role to deal with $N>2$, and the detailed derivation differs from \cite{vaskevicius2019implicit} a lot, especially for Lemma \ref{lemma: bounded error}, \ref{lemma: begin_noB} and \ref{lemma:neg growth}.

\subsection{Error Growth}
\label{sec: error}
\begin{lemma}
Consider the setting of updates given in equations \eqref{eq:updates-equation-using-b-p-notation}. Suppose that $\norm{\vec{e}_t}_\infty \leq \frac18 w_{min}^\star$ and there exists some $B\in\mathbb{R}$ such that for all $t$ we have $\norm{\vec{b}_t}_\infty + \norm{\vec{p}_t}_\infty \leq B$. Then, if $\eta \leq \frac{1}{12 (\wmax + B)}$ for any $t\geq 0$ we have 
\[
\norm{\vec{e}_t}_\infty \leq \norm{\vec{e}_0}_\infty \prod_{i=1}^{t-1} (1+2N\eta (\norm{\vec{b}_i}_\infty + \norm{\vec{p}_i}_\infty)\norm{\vec{e}_i}_\infty^{(N-2)/N})^N
\]
or in the other form,
\[
\norm{\vec{e}_{t+1}}_\infty \leq \norm{\vec{e}_t}_\infty  (1+2N\eta (\norm{\vec{b}_t}_\infty + \norm{\vec{p}_t}_\infty)\norm{\vec{e}_t}_\infty^{(N-2)/N})^N.
\]
\label{lemma:vector-error-growth}
\end{lemma}

\begin{proof}
From the equations above, we get 
\begin{align*}
    \id{S^c}\odot \vec{e}_{t+1} &= \id{S^c} \odot \vec{w}_t \odot(\id{} -2N\eta (\vec{s}_t - \wstar + \vec{p}_t + \vec{b}_t)\odot \vec{w}_t^{(N-2)/N})^N\\
    & = \id{S^c}\odot \vec{e}_t\odot(\id{S^c} - \id{S^c}2N\eta (\vec{s}_t - \wstar + \vec{p}_t + \vec{b}_t)\odot \vec{e}_t^{(N-2)/N})^N\\
    & = \id{S^c}\odot \vec{e}_t\odot(\id{} - 2N\eta ( \vec{p}_t + \vec{b}_t)\odot \vec{e}_t^{(N-2)/N})^N\\
\end{align*}
and hence 
\[
\norm{\id{S^c}\odot \vec{e}_{t+1}}_{\infty} \leq \norm{\vec{e}_t}_\infty (1+2N\eta(\norm{\vec{b}_t}_\infty + \norm{\vec{p}_t}_\infty)\norm{\vec{e}_t}_\infty^{(N-2)/N})^N.
\]

\end{proof}
When we have the bound for $\norm{\vec{b}_t}_\infty + \norm{\vec{p}_t}_\infty$, we can control the size of $\norm{\vec{e}_t}_\infty$ by the following lemma.

\begin{lemma}
\label{lemma: bounded error}
Let $(b_t)_{t\geq0}$ be a sequence such that for $t\geq0$ we have $|b_t|\leq B$ for some $B>0$. Let the step size satisfy $\eta\leq \frac{1}{4N(N-1)B x_0^{(N-2)/(2N)}}$ and consider a one-dimensional sequence $(x_t)_{t\geq0}$ given by
\begin{align*}
    0<& x_0<1,\\
    x_{t+1} =& x_t(1+2N\eta b_t x_t^{(N-2)/N})^N.
\end{align*}
Then for any $t < \frac{1}{8N(N-1)\eta B}\left(\frac{1}{x_0^{(N-2)/N}} -
    \frac{1}{x_0^{(N-2)/2N}}\right)$ we have
\[
x_t\leq \sqrt{x_0}.
\]
\begin{proof}
We start with studying the larger increasing rate of the updates,
\begin{align*}
    x_{t+1} &= x_t ( 1+ 2N\eta b_t x_t^{(N-2)/N})^N\\
    &\leq x_t ( 1+ 2N\eta B x_t^{(N-2)/N})^N \\
    &\leq x_t
    \left(1+\frac{2N^2\eta Bx_t^{(N-2)/N}}{1-2(N-1)N\eta x_t^{(N-2)/N}}\right)\\
    &\leq x_t(1+4N^2\eta Bx_t^{(N-2)/N}),
\end{align*}
where the second inequality is obtained by $(1+x)^r \leq 1+\frac{rx}{1-(r-1)x}$ for $x\in(0,\frac1{r-1})$, and the last inequality is by the requirement of step size $\eta$. 
Therefore, to achieve to some value $x_T$, the number of iterations needed is lower bounded as
\[
T \geq 
\sum_{t=0}^{T-1} \frac{x_{t+1}-x_t}{4N^2 \eta B x_t^{(2N-2)/N}}.
\]

We aim at the number of iterations for $\sqrt{x_0}$, and we denote $T$ as the maximal number of iterations, i.e. $x_T < \sqrt{x_0}$ and $x_{T+1}\geq \sqrt{x_0}$. Therefore, 
\[
\frac{\sqrt{x_0}-x_T}{4N^2 \eta B x_T^{(2N-2)/N}} \leq \frac{x_{T+1}-x_{T}}{4N^2 \eta B x_T^{(2N-2)/N}}\leq 1.
\]
And for $T$, we derive the lower bound as 
\begin{align*}
    T \geq \sum_{t=0}^{T-1} \frac{x_{t+1}-x_t}{4N^2 \eta B x_t^{(2N-2)/N}} &\geq \frac{1}{4N^2 \eta B}\sum_{t=0}^{T-1} \int_{x_t}^{x_{t+1}} \frac{1}{x^{(2N-2)/N}} dx\\
    &\geq  \frac{1}{4N^2 \eta B} \int_{x_0}^{x_T}\frac{1}{x^{(2N-2)/N}} dx\\
    &\geq  \frac{1}{4N^2 \eta B} \int_{x_0}^{\sqrt{x_0}}\frac{1}{x^{(2N-2)/N}} dx - \frac{1}{4N^2 \eta B} \int_{x_T}^{\sqrt{x_0}}\frac{1}{x^{(2N-2)/N}} dx\\
    & > \frac{1}{4N^2 \eta B} \left(-\frac{N}{2N-2} \frac{1}{x^{(N-2)/N}}\right)\Biggr|_{x_0}^{\sqrt{x_0}} - 1\\
    & = \frac{1}{8N(N-1)\eta B}\left(\frac{1}{x_0^{(N-2)/N}} -
    \frac{1}{x_0^{(N-2)/2N}}\right)-1.
\end{align*}
Therefore, we know that for any $t\leq \frac{1}{8N(N-1)\eta B}\left(\frac{1}{x_0^{(N-2)/N}} -
    \frac{1}{x_0^{(N-2)/2N}}\right)-1$, we have $x_t\leq\sqrt{x_0}$. Since in practice $t$ is chosen as an integer, without loss of generality, we simply the requirement as $t< \frac{1}{8N(N-1)\eta B}\left(\frac{1}{x_0^{(N-2)/N}} -
    \frac{1}{x_0^{(N-2)/2N}}\right)$.
\end{proof}
\end{lemma}

\subsection{Understanding 1-d Case}
\label{sec: 1d}
\subsubsection{Basic Setting}
In this subsection we analyze one-dimensional sequences with positive target corresponding to gradient descent updates without any perturbations. That is, $\vec{w}_t = \vec{u}_t^N$, $\frac1n \XtX = \matrixid$ and ignoring the error sequences $(\vec{b}_t)_{t\geq 0}$ and $(\vec{p}_t)_{t\geq 0}$. Hence, we will look at one-dimensional sequences of the form 
\begin{equation}
\begin{aligned}
0<x_0 &= \alpha^N < x^\star\\
x_{t+1} &= x_t(1-2N\eta(x_t-x^\star)x_t^{(N-2)/N})^N.
\end{aligned}
\end{equation}
\begin{lemma}[Iterates behave monotonically]
Let $\eta>0$ be the step size and suppose the updates are given by
\[
x_{t+1} = x_t(1-2N\eta(x_t-x^\star)x_t^{(N-2)/N})^N.
\]
Then the following holds
\begin{enumerate}
    \item If $0<x_0\leq x^\star$ and $\eta \leq \frac{1}{2N(2N-2)(x^\star)^{(2N-2)/N}}$ then for any $t>0$ we have $x_0 \leq x_{t-1}\leq x_t\leq x^\star$.
    \item If $x^\star \leq x_0 \leq \frac32 x^\star$ and $\eta \leq \frac{1}{6N^2 (x^\star)^{(2N-2)/N}}$ then for any $t\geq0$ we have $x^\star \leq x_t \leq x_{t-1} \leq \frac{3}{2}x^\star$.
\end{enumerate}
\label{lemma: mono_noB}
\end{lemma}
\begin{proof}
Note that if $x_0\leq x_t \leq x^\star$ then $x_t-x^\star \leq 0$ and hence $x_{t+1}\geq x_t$. Thus for the first part it is enough to show that for all $t\geq 0$ we have $x_t \leq x\leq x^\star$.

Assume for a contradiction that exists $t$ such that 
\begin{align*}
    x_0\leq x_t &\leq x^\star,\\
    x_{t+1} &> x^\star.
\end{align*}
Plugging in the update rule for $x_{t+1}$ we can rewrite the above as 
\begin{align*}
    x_t & \leq x^\star \\
    &< x_t(1-2N\eta(x_t-x^\star)x_t^{(N-2)/N})^N\\
    &\leq x_t\left(1+\frac1{2N-2} - \frac{x_t^{(2N-2)/N}}{(2N-2)(x^\star)^{(2N-2)/N}}\right)^N
\end{align*}
Letting $\lambda = \left(\frac{x_t}{x^\star}\right)^{(2N-2)/N}$, by our assumption we have $0<\lambda \leq 1$. The above inequality gives us 
\[
\left(\frac{1}{\lambda}\right)^\frac1{2N-2} <  1+\frac1{2N-2} - \frac1{2N-2} \lambda.
\]
And hence for $0<\lambda \leq 1$ we have $f(\lambda) \coloneqq \left(\frac{1}{\lambda}\right)^\frac1{2N-2} + \frac1{2N-2} \lambda < 1+1/(2N-2)$. Since for $0<\lambda<1$ we also have 
\[
f'(\lambda) = \frac1{2N-2} - \frac1{2N-2} \left( \frac1\lambda \right)^{\frac{1}{2N-2}+1}<0,
\]
so $f(\lambda) \geq f(1) = 1+1/(2N-2)$. This gives us the desired contradiction and concludes our proof for the first part.

We will now prove the second part. Similarly to the first part, we just need to show that for all $t\geq0$ we have $x_t \geq x^\star$. Suppose that $x^\star\leq x_t\leq \frac32 x^\star$ and hence we can write $x_t = x^\star(1+\gamma)$ for some $\gamma \in [0,\frac12]$. Then we have
\begin{align*}
x_{t+1} &= (1+\gamma)x^\star(1-2N\eta \gamma x^\star x_t^{(N-2)/N})^N \\
&\geq (1+\gamma)x^\star(1-3N\eta \gamma (x^\star)^{(N-2)/N})^N\\
&\geq x^\star (1+\gamma)\left(1-\frac{1}{2N}\gamma\right)^N\\
&\geq x^\star.
\end{align*}
The last inequality is obtained by letting $f(\gamma)\coloneqq(1+\gamma)\left(1-\frac{1}{2N}\gamma\right)^N$, we could get that 
\begin{align*}
f'(\gamma) &= \left(1-\frac{1}{2N}\gamma\right)^N - \frac{1}{2}(1+\gamma)\left(1-\frac{1}{2N}\gamma\right)^{N-1}\\
&=\left(1-\frac{1}{2N}\gamma\right)^{N-1}\left(\frac{1}{2}-\frac{1}{2} \gamma\right) >0.
\end{align*}
Hence, $f(\gamma)\geq f(0) = 1$ when $\gamma\in[0,\frac{1}{2}]$, which finishes the second part of our proof.
\end{proof}

\begin{lemma}[Iterates behaviour near convergence]
Consider the same setting as before. Let $x^\star > 0$ and suppose that $|x_0-x^\star|\leq\frac12 x^\star$. Then the following holds.
\begin{enumerate}
    \item If $x_0\leq x^\star$ and $\eta \leq \frac{1}{2N(2N-2)(x^\star)^{(2N-2)/N}}$, then for any $t\geq \frac{2}{\eta N^2 (x^\star)^\frac{2N-2}N}$ we have 
    \[
    0\leq x^\star - x_t \leq \frac{1}{2}|x_0-x^\star|.
    \]
    \item If $x^\star \leq x_0 \leq \frac32 x^\star$ and $\eta \leq \frac{1}{6N^2 (x^\star)^{(2N-2)/N}}$ then for any $t \geq \frac{1}{2N^2 \eta (x^\star)^{(2N-2)/N}}$ we have 
    \[
    0\leq x_t-x^\star \leq \frac12 |x_0-x^\star|.
    \]
\end{enumerate}
\label{lemma: near_noB}
\end{lemma}
\begin{proof}
Let us write $|x_0-x^\star|=\gamma x^\star$ where $\gamma \in [0,\frac12]$.

For the first part, we have $x_0 = (1-\gamma) x^\star$, we want to know how many steps $t$ are needed to halve the error, i.e.,
$$
x_t (1-2N\eta(x_t - x^\star) x_t^\frac{N-2}N))^N \geq (1-\frac\gamma2) x^\star.
$$

We have that
\begin{align*}
x_t (1-2N\eta(x_t - x^\star) x_t^\frac{N-2}N))^N 
&\geq x_t (1+2N\eta \frac\gamma2 x^\star ((1-\gamma)x^\star)^\frac{N-2}N))^N\\
&\geq x_0 (1+N\eta\gamma (1-\gamma)^\frac{N-2}N (x^\star)^\frac{2N-2}N))^{N t}
\end{align*}

It is enough to have
\begin{align*}
    &x_0 (1+N\eta\gamma (1-\gamma)^\frac{N-2}N (x^\star)^\frac{2N-2}N))^{N t}
    \geq (1-\frac\gamma2) x^\star \\
    \Rightarrow&
    (1-\gamma)(1+t N^2\eta\gamma (1-\gamma)^\frac{N-2}N (x^\star)^\frac{2N-2}N))
    \geq (1-\frac\gamma2)\\
    \Rightarrow& t \geq \left(\frac{1-\frac\gamma2}{1-\gamma}-1\right)\frac{1}{N^2\eta\gamma (1-\gamma)^\frac{N-2}N (x^\star)^\frac{2N-2}N}\\
    \Rightarrow& t \geq \frac{1}{2(1-\gamma)^\frac{2N-2}N N^2\eta  (x^\star)^\frac{2N-2}N}\\
    \Rightarrow& t \geq \frac{2}{\eta N^2 (x^\star)^\frac{2N-2}N}
\end{align*}
The last step is by $\gamma\in[0,\frac{1}{2}]$, we could obtain that $\frac{1}{2(1-\gamma)^\frac{2N-2}N} \leq \frac{1}{2 (1/2)^\frac{2N-2}{N}} \leq \frac{1}{2 (1/2)^2} \leq 2$. Therefore after $t\geq \frac{2}{\eta N^2 (x^\star)^\frac{2N-2}N}$, the error is halved.

To deal with the second part, we write $x_0 = x^\star(1+\gamma)$. We will use a similar approach as the one in the first part. If for some $x_t$ we have $x_t \leq (1+\gamma/2)x^\star$ we would be done. If $x_t > x^\star(1+\gamma/2)$ we have $x_{t+1} \leq x_t(1-2N\eta\frac\gamma2 x^\star (x^\star)^{(N-2)/N})^N$. Therefore, 
\begin{align*}
    &x_0(1-2N\eta\frac\gamma2 x^\star (x^\star)^{(N-2)/N})^{Nt} \leq x^\star (1+\gamma/2)\\
    \Longleftrightarrow &Nt \log(1-N\eta \gamma (x^\star)^{(2N-2)/N}) \leq \log \frac{x^\star(1+\gamma/2)}{x_0}\\
    \Longleftrightarrow& t \geq \frac1N \frac{\log \frac{x^\star(1+\gamma/2)}{x_0}}{\log(1-N\eta \gamma (x^\star)^{(2N-2)/N})}.
\end{align*}
We can deal with the term on the right hand side by noting that
\begin{align*}
    \frac1N \frac{\log \frac{x^\star(1+\gamma/2)}{x_0}}{\log(1-N\eta \gamma (x^\star)^{(2N-2)/N})} &= \frac1N \frac{\log \frac{1+\gamma/2}{1+\gamma}}{\log(1-N\eta \gamma (x^\star)^{(2N-2)/N})}\\
    &\leq \frac{1}{N} \frac{\left(\frac{1+\gamma/2}{1+\gamma} - 1 \right)/\left( \frac{1+\gamma/2}{1+\gamma}\right)}{-N\eta \gamma (x^\star)^{(2N-2)/N}} \\
    & = \frac1N \frac{-\frac\gamma2 / (1+\frac\gamma2)}{-N\eta \gamma (x^\star)^{(2N-2)/N}}\\
    &\leq \frac{1}{2N^2 \eta (x^\star)^{(2N-2)/N}}
\end{align*}
where the second line used $\log x\leq x-1$ and $\log x \geq \frac{x-1}x$. Note that both logarithms are negative.
\end{proof}

\begin{lemma}[Iterates at the beginning]
\label{lemma: begin_noB}
Consider the same setting as before. If $0<x_0\leq \frac12 x^\star$ and $\eta \leq \frac{x_0}{2N(2N-4)(x^\star)^{(3N-2)/N}}$, for any $t\geq\frac{3}{2N(N-2) \eta x^\star x_0^{(N-2)/N}}$, we will have $\frac12 x^\star \leq x_t\leq x^\star$.

\end{lemma}
\begin{proof}
We need to find a lower-bound on time $T$ which ensures that $x_T\geq \frac{x^\star}{2}$. At any time $t$, we have
\begin{align*}
    x_{t+1} = x_t(1-2N\eta(x_t-x^\star)x_t^{(N-2)/N})^N &\geq x_t(1-2N^2\eta(x_t-x^\star)x_t^{(N-2)/N}).\\
    x_{t+1} - x_t &\geq -2N^2\eta(x_t-x^\star)x_t^{(2N-2)/N}\\
    \frac{x_{t+1} - x_t}{2N^2\eta(x^\star-x_t)x_t^{(2N-2)/N}} &\geq 1\\
    \sum_{t=0}^{T-1}\frac{x_{t+1} - x_t}{2N^2\eta(x^\star-x_t)x_t^{(2N-2)/N}} &\geq \sum_{t=0}^{T-1}1 = T.
\end{align*}

Therefore, for $t$ that is larger than the left hand side, we have $x_t \geq \frac12 x^\star$. 
\begin{align}
    \sum_{t=0}^{T-1}\frac{x_{t+1} - x_t}{2N^2\eta(x^\star-x_t)x_t^{(2N-2)/N}} 
    &\leq \frac{1}{N^2 \eta x^\star}\sum_{t=0}^{T-1}\frac{x_{t+1} - x_t}{x_t^{(2N-2)/N}} \notag\\
    & = \frac{1}{N^2 \eta x^\star}\sum_{t=0}^{T-1}\int_{x_t}^{x_{t+1}}\frac{1}{x^{(2N-2)/N}} + \left(\frac{1}{x_t^{(2N-2)/N}} - \frac{1}{x^{(2N-2)/N}}\right)dx \notag\\
    &\leq \frac{1}{N^2 \eta x^\star}\sum_{t=0}^{T-1} \int_{x_t}^{x_{t+1}}\frac{1}{x^{(2N-2)/N}}dx \notag\\
    &+ \frac{1}{N^2 \eta x^\star}\max_{0\leq t \leq T-1}\left(\frac{1}{x_t^{(2N-2)/N}} - \frac{1}{x_{t+1}^{(2N-2)/N}}\right) (x_T-x_0) \notag\\
    &\leq \frac{1}{N^2 \eta x^\star} \int_{x_0}^{\frac12 x^\star}\frac{1}{x^{(2N-2)/N}}dx
    \label{eq:T2}\\
    &+ \frac{1}{N^2 \eta x^\star}\max_{0\leq t \leq T-1}\left(\frac{1}{x_t^{(2N-2)/N}} - \frac{1}{x_{t+1}^{(2N-2)/N}}\right) \left(\frac12 x^\star-x_0\right)
    \label{eq:approx error}\\
    &+\frac{1}{N^2 \eta x^\star} \frac{1}{(\frac12 x^\star)^{(2N-2)/N}} \left(x_T - \frac12 x^\star\right) \label{eq:margin error}
\end{align}
 For equation \eqref{eq:T2}, 
 \begin{align*}
     \frac{1}{N^2 \eta x^\star} \int_{x_0}^{\frac12 x^\star}\frac{1}{x^{(2N-2)/N}}dx 
     &\leq \frac{1}{N^2 \eta x^\star} \left(-\frac{N}{N-2} \frac{1}{x^{(N-2)/N}}\Biggr\vert^{\frac12 x^\star}_{x_0}\right)\\
     & = \frac{1}{N^2 \eta x^\star} \left( -\frac{N}{N-2} \frac{1}{(\frac12x^\star)^{(N-2)/N}} + -\frac{N}{N-2} \frac{1}{x_0^{(N-2)/N}}\right)\\
     & = \frac{1}{N(N-2) \eta x^\star} \left(
     \frac{1}{x_0^{(N-2)/N}} - \frac{2^{(N-2)/N}}{(x^\star)^{(N-2)/N}}\right).
 \end{align*}
 
 For equation \eqref{eq:approx error}, we first focus on 
 \[
 \frac{1}{x_t^{(2N-2)/N}} - \frac{1}{x_{t+1}^{(2N-2)/N}}.
 \]
 
 We have that 
 \begin{align*}
 x_{t+1} &= x_t (1-2N\eta(x_t-x^\star)x_t^{(N-2)/N})^N, \\
 \Rightarrow x_{t+1}^{(2N-2)/N} &= x_t^{(2N-2)/N} (1-2N\eta(x_t-x^\star)x_t^{(N-2)/N})^{2N-2}.
 \end{align*}
 
To deal with the multiplicative coefficient, with $\eta \leq \frac{1}{2N(2N-3)(x^\star)^{(2N-2)/N}}$ using the inequality $(1+x)^r \leq 1 + \frac{rx}{1-(r-1)x}$ where $x\in(0,\frac{1}{r-1})$, we obtain that
 \begin{align*}
     (1-2N\eta(x_t-x^\star)x_t^{(N-2)/N})^{2N-2} &\leq 
     (1 + 2N\eta (x^\star)^{(2N-2)/N})^{(2N-2)}\\
     & \leq
     1 + \frac{2N(2N-2)\eta(x^\star)^{(2N-2)/N}}{1-2N(2N-3)\eta(x^\star)^{(2N-2)/N}}\\
     &=\frac{1-2N\eta(x^\star)^{(2N-2)/N}}{1-2N(2N-3)\eta(x^\star)^{(2N-2)/N}}.
 \end{align*}
 Therefore,
 \begin{align*}
     \frac{1}{x_t^{(2N-2)/N}} - \frac{1}{x_{t+1}^{(2N-2)/N}} 
     &= \frac{1}{x_t^{(2N-2)/N}} - \frac{1}{x_t^{(2N-2)/N} (1-2N\eta(x_t-x^\star)x_t^{(N-2)/N})^{2N-2}}\\
     &= \frac{1}{x_t^{(2N-2)/N}}\left(1-\frac{1}{(1-2N\eta(x_t-x^\star)x_t^{(N-2)/N})^{2N-2}}\right)\\
     &\leq \frac{1}{x_t^{(2N-2)/N}} \left(1-\frac{1-2N(2N-3)\eta(x^\star)^{(2N-2)/N}}{1-2N\eta(x^\star)^{(2N-2)/N}}\right)\\
     &\leq  \frac{1}{x_t^{(2N-2)/N}} \frac{2N(2N-4)\eta(x^\star)^{(2N-2)/N}}{1-2N\eta(x^\star)^{(2N-2)/N}}\\
     &\leq \frac{1}{x_t^{(2N-2)/N}}2N(2N-4)\eta(x^\star)^{(2N-2)/N}\\
     &\leq \frac{1}{x_0^{(2N-2)/N}}2N(2N-4)\eta(x^\star)^{(2N-2)/N}.
 \end{align*}
 If we further require the step size satisfies $\eta \leq \frac{x_0}{2N(2N-4)(x^\star)^{(3N-2)/N}}$, we  have for equation \eqref{eq:approx error}, 
 \begin{align*}
     \frac{1}{N^2 \eta x^\star}\max_{0\leq t \leq T-1}\left(\frac{1}{x_t^{(2N-2)/N}} - \frac{1}{x_{t+1}^{(2N-2)/N}}\right) \left(\frac12 x^\star-x_0\right)
     &\leq 
     \frac{1}{N^2 \eta x^\star}
     \frac{1}{x_0^{(N-2)/N}x^\star}
     \left(\frac12 x^\star-x_0\right)\\
     &\leq 
     \frac{1}{2N^2 \eta x^\star}
     \frac{1}{x_0^{(N-2)/N}},
 \end{align*}
 which is with the same order with the result of equation \eqref{eq:T2}.
 
Combining the results from equations \eqref{eq:T2}, \eqref{eq:approx error}, \eqref{eq:margin error}, we obtain that
\begin{align*}
    T &\leq \frac{1}{N(N-2) \eta x^\star} \left(
     \frac{1}{x_0^{(N-2)/N}} - \frac{2^{(N-2)/N}}{(x^\star)^{(N-2)/N}}\right) + \frac{1}{2N^2 \eta x^\star}
     \frac{1}{x_0^{(N-2)/N}} \\
     &+ \frac{1}{N^2 \eta x^\star} \frac{1}{(\frac12 x^\star)^{(2N-2)/N}} \left(x_T - \frac12 x^\star\right) \\
     & \leq \frac{1}{N(N-2) \eta x^\star} \left(
     \frac{1}{x_0^{(N-2)/N}} - \frac{2^{(N-2)/N}}{(x^\star)^{(N-2)/N}} + \frac{1}{2x_0^{(N-2)/N}}
     + \frac{1}{(\frac12 x^\star)^{(N-2)/N}}\right)\\
     &\leq \frac{3}{2N(N-2) \eta x^\star x_0^{(N-2)/N}}.
\end{align*}
\end{proof}

\begin{lemma}[Overall iterates]
Consider the same setting as before. Fix any $\epsilon > 0$.
\label{lemma: all iterates}
\begin{enumerate}
    \item If $\epsilon < |x^\star - x_0| \leq \frac12 x^\star$ and $\eta \leq \frac{1}{6N^2 (x^\star)^{(2N-2)/N}}$ then for any $t\geq \frac{3}{\eta N^2 (x^\star)^\frac{2N-2}N} \log \frac{|x^\star - x_0|}{\epsilon}$ we have
    \[
    |x^\star - x_t| \leq \epsilon.
    \]
    \item If $0<x_0\leq \frac12 x^\star$ and $\eta \leq \frac{x_0}{2N(2N-4)(x^\star)^{(3N-2)/N}}$ then for any 
    $$
    t\geq \frac{3}{\eta N^2 (x^\star)^\frac{2N-2}N} \log \frac{|x^\star - x_0|}{\epsilon}
    +
    \frac{3}{2N(N-2) \eta x^\star x_0^{(N-2)/N}}
     $$
     we have 
     \[
     x^\star - \epsilon \leq x_t \leq x^\star.
     \]
\end{enumerate}
\end{lemma}
\begin{proof}
\begin{enumerate}
    \item To prove the first part we simply need apply Lemma \ref{lemma: near_noB} $\lceil \log_2 \frac{|x^\star - x_0|}{\epsilon}\rceil$ times. Hence after
    \[
    \frac{2\log_2 e}{\eta N^2 (x^\star)^\frac{2N-2}N} \log \frac{|x^\star - x_0|}{\epsilon} \leq 
    \frac{3}{\eta N^2 (x^\star)^\frac{2N-2}N} \log \frac{|x^\star - x_0|}{\epsilon}
    \]
    iterations we are done.
    
    \item For the second part, we simply combine the results from the first part and Lemma \ref{lemma: begin_noB}, it is enough to choose $t$ larger than or equal to
    \[\frac{3}{\eta N^2 (x^\star)^\frac{2N-2}N} \log \frac{|x^\star - x_0|}{\epsilon}
    +
    \frac{3}{2N(N-2) \eta x^\star x_0^{(N-2)/N}}.
    \]
    
\end{enumerate}
\end{proof}
\subsubsection{Dealing with Bounded Errors \texorpdfstring{$\vec{b}_t$}{Lg}}
In this subsection we extend the previous setting to handle bounded error sequences $(\vec{b}_t)_{t\geq 0}$ such that for any $t\geq0$ we have $\norm{\vec{b}_t}_\infty\leq B$ for some $B\in \mathbb{R}$. That is, we look at the following updates
\[
x_{t+1} = x_t(1-2N\eta(x_t-x^\star + b_t)x_t^{(N-2)/N})^N.
\]
Surely, if $B\geq x^\star$, the convergence to $x^\star$ is not possible. Hence, we will require $B$ to be small enough, with a particular choice $B\leq \frac1{5} x^\star$. For a given $B$, we can only expect the sequence $(x_t)_{t\geq0}$ to converge to $x^\star$ up to precision $B$. We would consider two extreme scenarios,
\begin{align*}
    x_{t+1}^+ &= x_t^+(1-2N\eta(x_t^+ -(x^\star-B))(x_t^+)^{(N-2)/N})^N,\\
    x_{t+1}^- &= x_t^-(1-2N\eta(x_t^- -(x^\star+B))(x_t^-)^{(N-2)/N})^N.
\end{align*}

\begin{lemma}[Squeezing iterates with bounded errors]
\label{lemma: mono_withB}
Consider the sequences $(x_t^-)_{t\geq 0}, (x_t)_{t\geq0}$ and $(x^+_t)_{t\geq0}$ as defined above with 
\[
0<x_0^- = x_0^+ = x_0 \leq x^\star + B
\]
If $\eta\leq\frac{1}{8N^2 (x^\star)^{(2N-2)/N}}$ then for all $t\geq 0$
\[
0\leq x_t^-\leq x_t\leq x_t^+ \leq x^\star + B.
\]
\end{lemma}

\begin{proof}
We will prove the claim by induction. The claim holds trivially for $t=0$. If $x_t^+\geq x_t$, we have
\begin{align*}
x^+_{t+1} &= x^+_t(1-2N\eta(x^+_t - (x^\star + B))(x^+_t)^\frac{N-2}{N})^N\\
&\geq x^+_t(1-2N\eta(x^+_t - (x^\star + B))x_t^\frac{N-2}{N})^N\\
\small (\triangle = x^+_t - x_t)\qquad
&=(x_t + \triangle)(1-2N\eta(x_t - x^\star + b_t)x_t^\frac{N-2}{N} \\
&+ 2N\eta(x^+_t - x_t - B - b_t)x_t^\frac{N-2}{N})^N\\
(m_t = 1-2N\eta(x_t- x^\star + b_t)x_t^\frac{N-2}{N})\qquad
&\geq (x_t + \triangle)(m_t - 2N\eta \triangle x_t^\frac{N-2}N)^N\\
&\geq (x_t + \triangle)(m_t - 2N\eta \triangle x_t^\frac{N-2}N)^N\\
&=x_t m_t^N + (x_t + \triangle)(m_t - 2N\eta \triangle x_t^\frac{N-2}N)^N - x_t m_t^N\\
&=x_t m_t^N + (x_t + \triangle)m_t^N\left(1 - \frac{2N\eta \triangle x_t^\frac{N-2}N}{m_t}\right)^N - x_t m_t^N.\\
\end{align*}
We aimed to show that $(x_t + \triangle)m_t^N(1 - 2N\eta \triangle x_t^\frac{N-2}N/m_t)^N - x_t m_t^N$ is positive. With $\eta \leq \frac{1}{4N(x^\star+B)(x^\star)^{(N-2)/N}}$, we can see $m_t\geq 1/2$ for all $t$ and 
\begin{align*}
(x_t + \triangle)m_t^N(1 - 2N\eta \triangle x_t^\frac{N-2}N/m_t)^N - x_t m_t^N 
&\geq 
(x_t + \triangle)m_t^N(1 - 4N\eta \triangle x_t^\frac{N-2}N)^N - x_t m_t^N\\
&\geq (x_t + \triangle)m_t^N(1 - 4N^2\eta \triangle x_t^\frac{N-2}N) - x_t m_t^N.
\end{align*}
The last inequality is obtained via $(1-x)^n\geq 1-nx$. If we further require $\eta \leq \frac{1}{8N^2 (x^\star)^{(2N-2)/N}}$, we obtain that
\begin{align*}
    (x_t + \triangle)m_t^N(1 - 4N^2\eta \triangle x_t^\frac{N-2}N) - x_t m_t^N
    &\geq 
    (x_t + \triangle)m_t^N
    \left(1 - \frac{1}{2x^\star} \triangle \right) - x_t m_t^N \\ 
    &\geq m_t^N\left(x_t + \triangle -\frac{x_t}{2 x^\star} \triangle - \frac{1}{2x^\star} \triangle^2 - x_t\right)\\
    &\geq m_t^N\triangle\left(1-\frac{x_t}{2x^\star} -\frac{\triangle}{2x^\star}\right)\\
        &\geq m_t^N\triangle\left(1-\frac12 -\frac12\right) \geq 0.
\end{align*}
Therefore, we obtain that
\[
x_{t+1}^+ \geq x_t m_t^N = x_{t+1}.
\]
For $x_{t}^-$, it follows a similar proof.
\end{proof}

\begin{lemma}[Iterates with bounded errors monotonic behaviour]
\label{lemma: iterates_withB}
Consider the previous setting with $B\leq \frac15 x^\star$, $\eta\leq \frac{1}{6N^2 (x^\star)^{(2N-2)/N}}$. Then the following holds
\begin{enumerate}
    \item If $|x_t-x^\star|>B$ then $|x_{t+1}-x^\star|<|x_t-x^\star|$.
    \item If $|x_t - x^\star|\leq B$ then $|x_{t+1}-x^\star| \leq B$.
\end{enumerate}
\end{lemma}
\begin{proof}
The choice of $B$ and step size $\eta$ ensures us to apply Lemma \ref{lemma: mono_noB} and Lemma \ref{lemma: mono_withB} to the sequences $(x^-_t)_{t\geq0}$ and $(x^+_t)_{t\geq0}$.
\end{proof}

\begin{lemma}[Iterates with $B$ near convergence]
\label{lemma: near with B}
Consider the setting as before. Then the following holds:
\begin{enumerate}
    \item If $\frac12 (x^\star - B) \leq x_0 \leq x^\star - 5B$ then for any $t\geq \frac{2}{\eta N^2 (x^\star)^\frac{2N-2}N}$we have 
    \[
    |x^\star - x_t| \leq \frac12 |x_0-x^\star|.
    \]
    \item If $x^\star + 4B < x_0 < \frac{6}{5} x^\star$ then for any $t\geq \frac{4}{\eta N^2 (x^\star)^\frac{2N-2}N}$ we have 
    \[
    |x^\star - x_t| \leq \frac12 |x_0-x^\star|.
    \]
\end{enumerate}
\end{lemma}

\begin{proof}
\begin{enumerate}
    \item To prove the first part, let us first apply Lemma \ref{lemma: near_noB} on $x^-_t$ twice, therefore for all 
    \[
    t \geq \frac{25}{4\eta N^2 (x^\star)^\frac{2N-2}N}\geq 2\frac{2}{\eta N^2 (x^\star- B)^\frac{2N-2}N}
    \]
    we have
    \begin{align*}
        0 &\leq (x^\star-B) - x^-_t \\
        &\leq \frac14 |x_0 - (x^\star-B)|\\
        &\leq \frac14 |x_0 - x^\star| + \frac14 B.
    \end{align*}
    When $x_t\leq x^\star$, from Lemma \ref{lemma: mono_withB} we have
    \begin{align*}
        0 &\leq x^\star - x_t \\
        &\leq x^\star - x_t^-\\
        &\frac14 |x_0-x^\star| + \frac54 B\\
        &\leq \frac12 |x_0 - x^\star|.
    \end{align*}
    
    When $x_t \geq x^\star$ then by Lemma \ref{lemma: mono_withB} we have
    \[
    0 \leq x_t - x^\star \leq B \leq \frac15 |x_0-x^\star|,
    \]
    where both last inequalities are from $x_0 \leq x^\star - 5B$.
    \item The second part follows a very similar proof for $x^+_t$, the number of iterations would be
    \[
    t\geq \frac{4}{\eta N^2 (x^\star)^\frac{2N-2}N} \geq 2\frac{2}{\eta N^2 (x^\star+ B)^\frac{2N-2}N}.
    \]
\end{enumerate}
\end{proof}

\begin{lemma}[Overall iterates with $B$]
Consider the same setting as before. Fix any $\epsilon > 0$, then the following holds
\begin{enumerate}
    \item If $B+\epsilon < |x^\star-x_0| \leq \frac15 x^\star$ then for any $t\geq \frac{15}{4\eta N^2 (x^\star)^\frac{2N-2}N} \log \frac{|x^\star - x_0|}{\epsilon}$ iterations we have $|x^\star - x_t | \leq B+\epsilon$.
    \item If $0<x_0\leq x^\star-B-\epsilon$ then for any 
    \[
    t\geq
    \frac{75}{16\eta N^2 (x^\star)^\frac{2N-2}N} \log \frac{|x^\star - x_0|}{\epsilon}
    +
    \frac{15}{8N(N-2)\eta x^\star x_0^{(N-2)/N}}
    \]
    we have $x^\star-B-\epsilon \leq x_t \leq x^\star+B$.
\end{enumerate}
\label{lemma: overall iterates with B}
\end{lemma}

\begin{proof}
\begin{enumerate}
    \item If $x_0 > x^\star + B$ then by Lemma \ref{lemma: mono_withB} and Lemma \ref{lemma: iterates_withB} we only need to show that $(x^+_t)_{t\geq0}$ hits $x^\star + B + \epsilon$ within the desired number of iterations. From the first part of Lemma \ref{lemma: all iterates}, we see that
    \[
    \frac{3}{\eta N^2 (x^\star+B)^\frac{2N-2}N} \log \frac{|x^\star+B- x_0|}{\epsilon}
    \leq 
    \frac{15}{4\eta N^2 (x^\star)^\frac{2N-2}N} \log \frac{|x^\star - x_0|}{\epsilon}
    \]
    iterations are enough, where we require $\frac{|x^\star-x_0|}{\epsilon}\geq \frac52$.
    \item The upper bound is obtained immediately from Lemma \ref{lemma: mono_withB}. For lower bound, we simply apply the second part of Lemma \ref{lemma: all iterates} to the sequence $(x^-_t)_{t\geq0}$ to get 
    \begin{align*}
    t&\geq 
    \frac{75}{16\eta N^2 (x^\star)^\frac{2N-2}N} \log \frac{|x^\star - x_0|}{\epsilon}
    +
    \frac{15}{8N(N-2)\eta x^\star x_0^{(N-2)/N}}
     \\
    &\geq \frac{3}{\eta N^2 (x^\star-B)^\frac{2N-2}N} \log \frac{|x^\star-B - x_0|}{\epsilon}
    +
    \frac{3}{2N(N-2)\eta (x^\star-B) x_0^{(N-2)/N}}\\
    \end{align*}
    to ensure the results we wanted.
\end{enumerate}
\end{proof}

\begin{lemma}
\label{lemma: second stage convergence}
Suppose the error sequences $(\vec{b}_t)_{t\geq0}$ and $(\vec{p}_t)_{t\geq0}$ satisfy the following for any $t\geq0$:
\begin{align*}
    \norm{\vec{b}_t\odot \id{S}} &\leq B,\\
    \norm{\vec{p}_t}_\infty &\leq \frac1{20} \norm{\vec{s}_t-\wstar}_\infty.
\end{align*}
Suppose that 
\[
20B < \norm{\vec{s}_0-\wstar}_\infty \leq\frac15 \wmin.
\]
Then for $\eta \leq \frac{1}{6N^2 (\wmax)^{(2N-2)/N}}$ and any $t\geq \frac{2}{\eta N^2 (\wmax)^{(2N-2)/N}}$ we have
\[
\norm{\vec{s}_t-\wstar}_\infty \leq \frac12 \norm{\vec{s}_0-\wstar}_\infty.
\]

\end{lemma}
\begin{proof}
Note that $\norm{\vec{b}_0}_\infty + \norm{\vec{p}_t}_\infty \leq \frac1{10} \norm{\vec{s}_0-\wstar}_\infty$. For any $i$ such that $|s_{0,i}-w_i^\star|\leq \frac12 \norm{\vec{s}_0-\wstar}_\infty$, Lemma \ref{lemma: iterates_withB} guarantees that for any $t\geq 0$ we have $|s_{t,i}-w_i^\star|\leq \frac12 \norm{\vec{s}_0-\wstar}_\infty$. On the other hand, for any $i$ such that $|s_{0,i}-w_i^\star| > \frac12 \norm{\vec{s}_0-\wstar}_\infty$ by Lemma \ref{lemma: near with B} we have $|s_{0,i}-w_i^\star|\leq \frac12 \norm{\vec{s}_0-\wstar}_\infty$ for any $t\geq \frac{2}{\eta N^2 (\wmax)^{(2N-2)/N}}$ which concludes the proof.
\end{proof}

\subsection{Dealing with Negative Targets}
\label{sec: neg}

\begin{lemma}
Let $x_t = u^N- v^N$ and $x^\star\in\mathbb{R}$ be the target such that $|x^\star|>0$. Suppose the sequences $(u_t)_{t\geq0}$ and $(v_t)_{t\geq0}$ evolve as follows
\begin{align*}
0<u_0=\alpha, 
\quad
&
u_{t+1} = u_t(1-2N\eta(x_t-x^\star+b_t)u_t^{N-2}),\\
0<v_0=\alpha, 
\quad
&
v_{t+1} = v_t(1+2N\eta(x_t-x^\star+b_t)v_t^{N-2}),
\end{align*}
where $\alpha \leq (2-2^\frac{N-2}N)^\frac{1}{N-2}|x^\star|^{1/N}$ and there exists $B>0$ such that $|b_t|\leq B$
and $\eta\leq \frac{\alpha}{4N(N-2)(x^\star+B)x^\star}$. 
Then the following holds:
For any $t\geq0$ we have
\begin{itemize}
    \item If $x^\star > 0$ and $u_t^N \geq x^\star$, 
    then $v_t^N \leq \frac12 \alpha^N$.
    \item If $x^\star < 0$ and $v_t^N \geq |x^\star|$, 
    then $u_t^N \leq \frac12 \alpha^N$.
\end{itemize}
\label{lemma:neg growth}
\end{lemma}
\begin{proof}
Let us assume $x^\star>0$ first and prove the first statement.
From the updating equation, we obtain that
\[
\frac{u_{t+1}-u_t}{u_t^{N-1}} = -2N\eta(x_t-x^\star+b_t).
\]
Therefore, 
\begin{align*}
    \sum_{i=0}^t -2N\eta(x_i-x^\star+b_i) 
    &= \sum_{i=0}^t \frac{u_{i+1}-u_i}{u_i^{N-1}}\\
    &\geq \sum_{i=0}^t \int_{u_i}^{u_{i+1}} \frac{1}{u^{N-1}} du\\
    & = \int_{u_0}^{u_t} \frac1{u^{N-1}}du \\
    &= (2-N)(u_t^{2-N} - u_0^{2-N}).
\end{align*}

When $u_t^N \geq x^\star$, we have that $u_t^{2-N} \leq (x^\star)^{(2-N)/N}$. Therefore,
\[
\sum_{i=1}^t -2N\eta(x_i-x^\star+b_i) 
\geq (2-N)(u_t^{2-N}-u_0^{2-N}).
\]

Similarly for $v_t$, we have
\begin{align*}
    \sum_{i=1}^t 2N\eta(x_i-x^\star+b_i)
    &= 
    \sum_{i=0}^t \frac{v_{i+1}-v_i}{v_i^{N-1}}\\
    &\geq (2-N)(v_t^{2-N} - v_0^{2-N}).
\end{align*}

Therefore, we have that
\begin{align*}
    &(N-2)((x^\star)^\frac{2-N}{N} - \alpha^{2-N})
    \geq (2-N)(v_t^{2-N} - \alpha^{2-N}).\\
    &\Longrightarrow 
    (\alpha^{2-N}-(x^\star)^\frac{2-N}N) + \alpha^{2-N} \leq v_t^{2-N}.\\
    &\Longrightarrow
    v_t \leq \left(\frac{1}{2\alpha^{2-N} - (x^\star)^\frac{2-N}N}\right)^{\frac{1}{N-2}}\\
    &\Longrightarrow 
    v_t \leq \left(\frac{1}{2 - \alpha^{N-2}/(x^\star)^\frac{N-2}N}\right)^{\frac{1}{N-2}}\alpha\\
    &\Longrightarrow 
    v_t \leq \left(\frac{1}{2-(2-2^{\frac{N-2}{N}})}\right)^\frac{1}{N-2}\alpha\\
    &\Longrightarrow
    v_t \leq 2^\frac{1}{N} \alpha.
\end{align*}

For $x^\star<0$, we obtain a similar result by symmetry.
\end{proof}

\begin{lemma}
\label{lemma: negative}
Let $x_t = x_t^+-x_t^-$ and $x^\star\in\mathbb{R}$ be the target such that $|x^\star|>0$. Suppose the sequences $(x_t^+)_{t\geq0}$ and $(x_t^-)_{t\geq0}$ evolve as follows
\begin{align*}
    0<x_0^+=\alpha^N \leq 2^3(2^\frac1N-1)^\frac{N}{N-2}|x^\star|,\quad 
    x_{t+1}^+ = x_t^+(1-2N\eta(x_t-x^\star+b_t)(x_t^+)^{(N-2)/N})^N,\\
    0<x_0^-=\alpha^N \leq 2^3(2^\frac1N-1)^\frac{N}{N-2}|x^\star|,\quad 
    x_{t+1}^- = x_t^-(1+2N\eta(x_t-x^\star+b_t)(x_t^-)^{(N-2)/N})^N,\\
\end{align*}
and that there exists $B > 0$ such that $|b_t|\leq B$ and $\eta \leq \frac{1}{8N(x^\star+B)(x^\star)^{(N-2)/N}}$.
Then the following holds:
For any $t\geq0$ we have
    \begin{itemize}
        \item If $x^\star > 0 $ then $x_t^- \leq \alpha^N \Pi_{i=0}^{t-1}(1+2N\eta|b_t|(x_i^-)^{(N-2)/N})^N$.
        \item If $x^\star < 0 $ then $x_t^+ \leq \alpha^N \Pi_{i=0}^{t-1}(1+2N\eta|b_t|(x_i^+)^{(N-2)/N})^N$.
    \end{itemize}
\end{lemma}
\begin{proof}

Assume $x^\star > 0$ and fix any $t\geq0$. Let $0\leq s \leq t$ be the largest $s$ such that $x_s^+ > x^\star$. If no such $s$ exists we are done immediately. If $s=t$ then by the first part we have $x^-_t\leq\alpha^N$ and we are done.

If $s<t$, by Lemma \ref{lemma:neg growth}, we have $x_s^- \leq \frac12 \alpha^N$. From the requirement of initialization, we have
\begin{align*}
    (1+2N\eta(x_s^+-x_s^- - x^\star + b_s)(x_s^-)^{(N-2)/N})^N
    \leq 
    \left(1+\frac{(\frac12 \alpha^N)^{\frac{N-2}{N}}}{4 (x^\star)^{\frac{N-2}{N}}}\right)^N \leq (1+2^{\frac1N}-1)^N = 2.
\end{align*}

Therefore 

\begin{align*}
    x_t^- 
    &= x_s^- \prod_{i=s}^{t-1} (1+2N\eta(x_i^+-x_i^- - x^\star + b_i)(x_i^-)^{(N-2)/N})^N \\
    &= \frac12 \alpha^N \cdot 2
    \prod_{i=s+1}^{t-1} (1+2N\eta(x_i^+-x_i^- - x^\star + b_i)(x_i^-)^{(N-2)/N})^N \\
    &\leq \alpha^N \prod_{i=s+1}^{t-1} (1+2N\eta|b_i|(x_i^-)^{(N-2)/N})^N.
\end{align*}
This completes the proof for $x^\star >0$. It follows a similar proof for the case $x^\star < 0$.
\end{proof}

\section{Proof of Propositions and Technical Lemmas}
In this section, we provide the proof for the propositions and technical lemmas mentioned in Appendix~\ref{sec:proof-for-non-negative}.
\label{sec: proof of prop}

\subsection{Proof of Proposition 1}

By the assumptions on $(\vec{b}_t)_{t\geq0}$ and $(\vec{p}_t)_{t\geq0}$, we obtain that
\begin{align*}
    \norm{\vec{b}_t}_\infty &\leq C_b\zeta - \alpha^{N/4},\\
\norm{\vec{p}_t}_\infty &\leq \frac{C_\gamma}{\wmax/\zeta}\norm{\vec{s}_t-\wstar}_\infty \leq \frac{C_\gamma}{\wmax/\zeta} \wmax \leq C_\gamma \zeta.\\
\end{align*}
Choose $C_b$ and $C_\gamma$ such that $C_b+C_\gamma\leq 1/40$. Therefore, we have
\[
B 
\leq 
\norm{\vec{b}_t}_\infty + \norm{\vec{p}_t}_\infty + \alpha^{N/4}
\leq 
(C_b+C_\gamma)\zeta \leq \frac{1}{40} \zeta.
\]

For any $j$ such that $w_j^\star\geq \frac12 \zeta$, we have that $B\leq \frac1{20} w_j^\star$. 
Therefore, by applying Lemma \ref{lemma: overall iterates with B}, we know when 
\begin{align*}
t\geq 
    \frac{75}{16\eta N^2 \zeta^{(2N-2)/N}} \log \frac{|\wmax - \alpha^N|}{\epsilon}
    +
    \frac{15}{8N(N-2)\eta \zeta \alpha^{(N-2)}}=T_1,
\end{align*}
we have $|w_{j,t} - w_j^\star|\leq \zeta$. 

On the other hand, for any $j$ such that $w_j^\star\leq \frac12\zeta$, $w_{j,t}$ will stay in $(0,w_j^\star+\frac1{40} \zeta]$ maintaining $|w_{j,t} -w_j^\star|\leq \zeta$ as required.

By Lemma \ref{lemma: bounded error}, we have that $\norm{\vec{e}_t}_\infty\leq \alpha^{N/2}$ up to 
\[
T_2 = \frac{5}{N(N-1)\eta \zeta}\left(\frac{1}{\alpha^{N-2}} - \frac{1}{\alpha^\frac{N-2}2} \right).
\]

From our choice of initialization $\alpha$, we can see that $T_1\leq T_2$ is ensured. To see this,
\begin{equation}
\begin{aligned}
& 
\alpha \leq
\left(\frac18\right)^{2/(N-2)} \wedge
\left(\frac{\zeta^{(N-2)/N}}{\log \frac{\wmax}{\epsilon}} \right)^{2/(N-2)}
\\
\Longrightarrow& 
\alpha^{(N-2)/2} \leq \frac18 \wedge
\frac{16\zeta^{(N-2)/N}}{15\log \frac{\wmax}{\epsilon}} 
\\
\Longrightarrow&
4\alpha^{(N-2)/2}\left(\frac{15}{16}\alpha^{(N-2)/2} \log \frac{\wmax}{\epsilon}
    +
\zeta^{(N-2)/N}\right)
\leq 
\zeta^{(N-2)/N}
\\
\Longrightarrow&
\frac{15}{2\zeta^{(N-2)/N}} \log \frac{\wmax}{\epsilon}
    +
\frac{6}{ \alpha^{(N-2)}} \leq 8\left(\frac{1}{\alpha^{(N-2)}}-\frac{1}{\alpha^{(N-2)/2}}\right)
\\
\Longrightarrow& 
\frac{75}{16\eta N^2 \zeta^{(2N-2)/N}} \log \frac{|\wmax - \alpha^N|}{\epsilon}
    +
\frac{15}{8N(N-2)\eta \zeta \alpha^{(N-2)}} \leq \frac{5}{N(N-1)\eta\zeta}\left(\frac{1}{\alpha^{(N-2)}}-\frac{1}{\alpha^{(N-2)/2}}\right) \\
\Longrightarrow& T_1\leq T_2.
\end{aligned}
\label{eq:T1-smaller-than-T2}
\end{equation}
\qed 

\subsection{Proof of Proposition 2}
By Lemma \ref{lemma: bounded error}, with the choice of $B=\frac{1}{200}\wmin$, we can maintain $\norm{\vec{e}_t}_\infty \leq \alpha^{N/4}$ for at least another 
\[
t\leq \frac{25}{N(N-1)\eta \wmin}\left(\frac{1}{\alpha^{(N-2)/2}} -
    \frac{1}{\alpha^{(N-2)/4}}\right) = T_4.
\]

Now we consider to further reduce $\norm{\vec{s}_t-\wstar}_\infty$ from $\frac15 \wmin$ to $\norm{\frac1n \Xt\xxi}_\infty \vee \epsilon$. Let $B_i \coloneqq (\vec{b}_t)_i$ and $B \coloneqq \max_{j\in S} B_j$. 

We first apply Lemma \ref{lemma: second stage convergence} for $\log_2 \frac{\wmin}{100(B\vee\epsilon)}$ times, the total number of iterations for this step would be 
\[
\frac{2}{\eta N^2 (\wmin)^{(2N-2)/N}}\log_2 \frac{\wmin}{100(B\vee\epsilon)}.
\]
After that we have $\norm{\vec{s}_t-\wstar}_\infty < 20(B\vee\epsilon)$ and so $\norm{\vec{p}_t}_\infty < k\mu\cdot 20 (B\vee\epsilon)$. Hence, for any $i\in S$ we have 
\[
\norm{\vec{b}_t\odot \id{i}}_\infty + \norm{\vec{p}_t}_\infty \leq B_i + k\mu 20(B\vee \epsilon).
\]
Then we further apply Lemma \ref{lemma: overall iterates with B} for each coordinate $i\in S$ to obtain that
\[
|w_{i,t}-w^\star_i|\lesssim \left|\frac1n(\Xt\xxi)_i\right|
\vee
k\mu\norm{\frac1n \Xt\xxi\odot\id{S}}_\infty\vee\epsilon.
\]
the number of iterations needed for this step is $\frac{15}{4\eta N^2 (\wmin)^{(2N-2)/N}}\log \frac{\wmin}{\epsilon}$.

Therefore the total number of iterations needed to further reduce $\norm{\vec{s}_t-\wstar}_\infty$ is 
\begin{align*}
T_3 &= \frac{6}{\eta N^2 (\wmin)^{(2N-2)/N}}\log \frac{\wmin}{\epsilon}\\
&\geq\frac{2}{\eta N^2 (\wmin)^{(2N-2)/N}}\log_2 \frac{\wmin}{100(B\vee\epsilon)}
+ \frac{15}{4\eta N^2 (\wmin)^{(2N-2)/N}}\log \frac{\wmin}{\epsilon}.
\end{align*}
Since $T_3$ is no longer related to $\alpha$, we can easily ensure $T_3\leq T_4$ with some mild upper bound on $\alpha^{(N-2)/4} \leq \frac{(\wmin)^{(N-2)/N}}{\log \frac{\wmin}{\epsilon}} \wedge 1/2$.

\begin{equation}
\begin{aligned}
&    \alpha^{(N-2)/4} \leq \frac{(\wmin)^{(N-2)/N}}{\log \frac{\wmin}{\epsilon}} \wedge 1/2\\
\Longrightarrow&    \alpha^{(N-2)/4}\left(\alpha^{(N-2)/4}\log\frac{\wmin}{\epsilon} + (\wmin)^{(N-2)/N}\right) \leq (\wmin)^{(N-2)/N}\\
\Longrightarrow& \alpha^{(N-2)/2}\log\frac{\wmin}{\epsilon}\leq (\wmin)^{(N-2)/N}-\alpha^{(N-2)/4}(\wmin)^{(N-2)/N}\\
\Longrightarrow& \frac{1}{(\wmin)^{(N-2)/N}}\log\frac{\wmin}{\epsilon} \leq \frac{1}{\alpha^{(N-2)/2}} - \frac{1}{\alpha^{(N-2)/4}}\\
\Longrightarrow& \frac{6}{\eta N^2 (\wmin)^{(2N-2)/N}}\log\frac{\wmin}{\epsilon} \leq \frac{25}{\eta N (N-1)\wmin}\left(\frac{1}{\alpha^{(N-2)/2}} - \frac{1}{\alpha^{(N-2)/4}}\right)\\
\Longrightarrow& T_3\leq T_4.
\end{aligned}
\label{eq:T3-smaller-than-T4}
\end{equation}
\qed 

\subsection{Proof of Technical Lemmas}

\emph{Proof of Lemma \ref{lemma:max-eroor}.} 
Since $\frac{1}{\sqrt{n}}\X$ is with $\ell_2$-normalized columns and satisfies $\mu$-coherence, where $0\leq\mu\leq1$,
\[
\left|\left(\frac{1}{n}\XtX\right)_{i,j}\right|
=
\left|\left(\frac{1}{\sqrt{n}}\X_i \right)^\top \left(\frac{1}{\sqrt{n}}\X_j\right)\right| \leq \max\{1,\mu\} \leq 1.
\]
Therefore, for any $\vec{z}\in\mathbb{R}^p$, 
\[
\norm{\frac1n \XtX \vec{z}}_\infty
\leq p\norm{\vec{z}}_\infty.
\]
\qed

\emph{Proof of Lemma \ref{lemma:incoherence-assumption-pt}.}
It is straightforward to verify that for any $i\in\{1,\ldots,p\}$, 
\[
\left|\left(\frac1n \XtX \vec{z}\right)_i - \vec{z}_i \right| 
\leq 
k \mu \norm{\vec{z}}_\infty.
\]
Therefore,
  \[
    \norm{\left( \frac{1}{n} \XtX - \matrix{I} \right)\vec{z}}_{\infty}
    \leq
    k \mu \norm{\vec{z}}_{\infty}.
  \]
\qed

\emph{Proof of Lemma \ref{lemma:bounding-max-noise}.}
Since the vector $\vec{\xxi}$ are made of independent $\sigma^2$-subGaussian random variables and any column $\X_i$ of $\X$ is $\ell_2$-normalized, i.e. $\norm{\frac{1}{\sqrt{n}} \X_i} = 1$, the random variable $\frac{1}{\sqrt{n}} (\Xt \vec{\xxi})_i$ is still $\sigma^2$-subGaussian.

It is a standard result that for any $\epsilon > 0$,
\[
\mathbb{P}\left(\norm{\frac{1}{\sqrt{n}} \Xt \vec{\xxi}}_\infty > \epsilon\right)
\leq 
2p \exp\left(-\frac{\epsilon^2}{2\sigma^2}\right).
\]
Setting $\epsilon=2\sqrt{2\sigma^2\log(2p)}$, with probability at least $1-\frac{1}{8p^3}$ we have
\[
  \norm{\frac1n \Xt \vec{\xxi}}_\infty 
  \leq \frac{1}{\sqrt{n}}2\sqrt{\sigma^2\log(2p)}
  \lesssim \sqrt{\frac{\sigma^{2} \log p}{n}}.
\]
\qed

\section{Proof of Theorems in Section \ref{sec:main-res}}
In this section, we provide the proof for all results we mentioned in Section~\ref{sec:main-res}.
\label{sec:proof-of-main-res}

\subsection{Proof of Theorem \ref{thm:general}}
\begin{proof}

Now let us consider the updates in equation \eqref{eq:updates-equation-using-b-p-notation}. The major idea is to show that the results in Theorem \ref{thm:non-negative} can be easily generalized with the lemmas we developed in Section \ref{sec: neg}.

Let us denote 
\[
\Psi(\wmin,N) \coloneqq
(2-2^\frac{N-2}{N})^\frac{1}{N-2}(\wmin)^{\frac1N}
\wedge 
2^{\frac3N}(2^\frac{1}{N}-1)^\frac{1}{N-2} (\wmin)^{\frac1N}.
\]

We set 
\[
\alpha \leq \left(\frac{\epsilon}{p+1}\right)^{4/N}
\wedge \Phi(\wmax,\wmin,\epsilon,N) 
\wedge \Psi(\wmin,N).
\]

Under the same requirements on other parameters with Theorem \ref{thm:non-negative}, we satisfy the conditions of Lemma \ref{lemma:vector-error-growth}, Lemma \ref{lemma: bounded error} and Lemma \ref{lemma: negative}. From these lemmas, we could maintain that
\begin{align*}
    w_j^\star >0 &\Longrightarrow 0\leq w_t^- \leq \alpha^{N/4},\\
    w_j^\star <0 &\Longrightarrow 0\leq w_t^+ \leq \alpha^{N/4},
\end{align*}
up to $T_2+T_4$ as defined in Proposition \ref{prop: first stage} and \ref{prop: second stage}.

Consequently, for $w_j^\star >0$ we can ignore $(w_{j,t}^-)_{t\geq 0}$ by treating as a part of bounded error $b_t$. The same holds for sequence $(w_{j,t}^+)_{t\geq0}$ when $w_j^\star <0$. Then, for $w_j^\star > 0$ the sequence $(w_{j,t}^+)$ evolves as follows
\[
w^+_{j,t+1} = w_{j,t}^+(1-2N\eta(w_{j,t}^+-w_j^\star+(b_{j,t}-w_{j,t}^-)+p_{j,t})(w_{j,t}^+)^{(N-2)/2})^N.
\]

The $b_{j,t}-w^-_{j,t}$ explains why we need $\norm{\vec{b}_t}_\infty + \alpha^{N/4}\leq C_b \zeta$ in Proposition \ref{prop: first stage}. For $w_j^\star>0$, we follow the exact proof structure with Theorem \ref{thm:non-negative} with treating $(w_{j,t}^-)_{t\geq 0}$ as a part of bounded error. For $w_j^\star<0$ it follows the same argument by switching $w_t^+$ and $w_t^-$.

Therefore, we could closely follow the proof of Theorem \ref{thm:non-negative} to generalize the result from non-negative signals to general signals. The result remains unchanged as well as the number of iterations requirement in equation \eqref{eq:T1+T3} and  \eqref{eq:T2+T4}. With the choice of $C_b=\frac{1}{100}$ in the proof of Theorem \ref{thm:non-negative}, recall that
\[
\zeta = \frac15 \wmin \vee 
200\norm{\frac1n \Xt \xxi}_\infty \vee
200\epsilon,
\]
and define the indicator function with $A$ as the event $\{\frac15 \wmin 
> 
200\norm{\frac1n \Xt \vec{\xxi}}_\infty 
\vee
200\epsilon\}$,
\[
\mathbbm{1}{(A)}
=
\begin{cases}
1,\quad\text{when}\quad  
\frac15 \wmin 
> 
200\norm{\frac1n \Xt \vec{\xxi}}_\infty 
\vee
200\epsilon,\\
0,\quad\text{when}\quad  
\frac15 \wmin 
\leq
200\norm{\frac1n \Xt \vec{\xxi}}_\infty 
\vee
200\epsilon.
\end{cases}
\]
We now define that 
\begin{equation}
\begin{aligned}
T_l(\wstar,\alpha,N,\eta,\zeta,\epsilon)
&\coloneqq\frac{75}{16\eta N^2 \zeta^{(2N-2)/N}} \log \frac{|\wmax - \alpha^N|}{\epsilon}
+
\frac{15}{8N(N-2)\eta \zeta \alpha^{(N-2)}}\\
&+
\frac{6}{\eta N^2 (\wmin)^{(2N-2)/N}}\log \frac{\wmin}{\epsilon} 
\mathbbm{1}{(A)},\\
T_u(\wstar,\alpha,N,\eta,\zeta,\epsilon)
&\coloneqq\frac{5}{N(N-1)\eta\zeta}\left(\frac{1}{\alpha^{(N-2)}}-\frac{1}{\alpha^{(N-2)/2}}\right)\\
&+
\frac{25}{N(N-1)\eta \wmin}
\left(\frac{1}{\alpha^{(N-2)/2}} -
\frac{1}{\alpha^{(N-2)/4}}\right)
\mathbbm{1}{(A)}.
\end{aligned}
\label{eq:Tl-and-Tu}
\end{equation}
The error bound \eqref{eq:non-negative-error-bound} holds for any $t$ such that 
\[
T_l(\wstar,\alpha,N,\eta,\zeta,\epsilon) \leq t \leq T_u(\wstar,\alpha,N,\eta,\zeta,\epsilon).
\] 
The equation \eqref{eq:T1-smaller-than-T2} and \eqref{eq:T3-smaller-than-T4} ensure that it is not a null set.

Thus, we finish generalizing Theorem \ref{thm:non-negative} to general signals with an extra requirement $\Psi(\wmin,N)$ on the initialization $\alpha$.

For the case $k=0$, i.e., $\wstar = \vec{0}$, we set $\wmin=0$ and
\[
\alpha \leq \left( \frac{\epsilon}{p+1}\right)^{4/N}.
\]

Conditioning on $\norm{\vec{e}_t}_\infty \leq \alpha^{N/4}$, we still have that
\[
\norm{\vec{b}_t}_\infty + \alpha^{N/4} 
\leq 
p\alpha^{N/4} + \norm{\frac1n \Xt \xxi}_\infty + \alpha^{N/4} 
\leq 
2 \left(\norm{\frac1n \Xt \vec{\xxi}}_\infty \vee \epsilon \right) 
\leq C_b \zeta \leq \frac{1}{40}\zeta.
\]

Therefore, by Lemma \ref{lemma: bounded error}, for $\eta \leq \frac{1}{N(N-1) \zeta \alpha^{(N-2)/2}}$, we ensure $\norm{\vec{e}_t}_\infty \leq \alpha^{N/4}$ up to $
\frac{5}{N(N-1)\eta \zeta}
\left(\frac{1}{\alpha^{N-2}} - \frac{1}{\alpha^{(N-2)/2}}\right),
$
which agrees to the definition of $T_u(\wstar,\alpha,N,\eta,\zeta,\epsilon)$ in this case.
\end{proof}

\subsection{Proof of Corollary \ref{cor:l2-error}}
Since $\vec{\xxi}$ is made of independent $\sigma^2$-sub-Gaussian entries, by Lemma \ref{lemma:bounding-max-noise} with probability $1-1/(8p^3)$ we have
\[
\norm{\frac1n \Xt \vec{\xxi}}_\infty \leq 2 \sqrt{\frac{2\sigma^2 \log(2p)}{n}}.
\]
Hence, letting $\epsilon = 2 \sqrt{\frac{2\sigma^2 \log(2p)}{n}}$, we obtain that
\[
\norm{\vec{w}_t - \wstar}_2^2 \lesssim 
\sum_{i\in S} \epsilon^2 + \sum_{i\notin S}\alpha^{N/2}
\leq k\epsilon^2 + (p-k) \frac{\epsilon^2}{(p+1)^2} \lesssim \frac{k\sigma^2 \log p}{n}.
\]
\qed

\subsection{Proof of Theorem \ref{thm:informal-early-stopping}}
\label{sec:early-stopping}
We now state Theorem \ref{thm:informal-early-stopping} formally as below.
\begin{theorem}
Let $T_1$, $T_2$, $T_3$ and $T_4$ be the number of iterations defined in Proposition \ref{prop: first stage} and Proposition \ref{prop: second stage}.
Suppose $\zeta\geq1$, $\wmax\geq 1$ and the initialization $\alpha \leq \exp(-5/3)$, fixing $\alpha$ and $\eta$ for all $N$,
both $T_2-T_1$ and $T_4-T_3$ have a tight lower bound that is increasing as $N$ increases $(N>2)$.
\label{thm: early stopping}
\end{theorem}
\begin{proof}
We observe first that under the assumption $\zeta\geq1$ and $\wmax\geq 1$, $\frac{75}{16\eta N^2 \zeta^{(2N-2)/N}} \log \frac{|\wmax - w_0|}{\epsilon}$ and $T_3 = \frac{6}{\eta N^2 (\wmax)^{(2N-2)/N}}\log \frac{\wmin}{\epsilon}$ are decreasing as $N$ increases.

For the rest part of $T_2-T_1$, we will be showing that a lower bound of that is increasing as $N$ increases. As $T_2-T_1$ is by design a lower bound of the ``true'' early stopping window, the lower bound we get here is tight for $T_2-T_1$ and is treated as equivalent to $T_2-T_1$ to indicate the monotonicity of the "true" early stopping window.

\begin{align*}
    \frac{5}{N(N-1)\eta\zeta}\left(\frac{1}{\alpha^{(N-2)}}-\frac{1}{\alpha^{(N-2)/2}}\right)
    &-
    \frac{15}{8N(N-2)\eta \zeta \alpha^{(N-2)}} \\
    &\geq 
    \frac{5}{4N(N-1)\eta\zeta}\left(\frac{1}{\alpha^{(N-2)}}-\frac{4}{\alpha^{(N-2)/2}}\right)
\end{align*}
Denote 
\[
f(N) = \frac{1}{N(N-1)}\left(\frac{1}{\alpha^{(N-2)}} -
    \frac{4}{\alpha^{(N-2)/2}}\right).
\]

Therefore,
\begin{align*}
    f'(N) &= \frac{-(2N-1)}{N^2(N-1)^2}\left(\frac{1}{\alpha^{(N-2)}} -
    \frac{4}{\alpha^{(N-2)/2}}\right) \\
    & + \frac{1}{N(N-1)}(-\log \alpha)\left(\frac{1}{\alpha^{(N-2)}} -
    \frac{4}{2\alpha^{(N-2)/2}}\right)\\
    &= \frac{-(2N-1)-(N-1)N\log \alpha }{2N^2 (N-1)^2}\left(\frac{1}{\alpha^{(N-2)}} -
    \frac{2}{\alpha^{(N-2)/2}}\right)\\
    &+\frac{2N-1}{N^2(N-1)^2}\frac{2}{\alpha^{(N-2)/4}}
\end{align*}
Note that the second term is always positive, we just need to show the first term is positive.
\begin{align*}
    -(2N-1)-(N-1)N\log \alpha &\geq 0,\\
    \frac{1}{\alpha^{(N-2)}} -
    \frac{2}{\alpha^{(N-2)/2}} &\geq 0,\\
\end{align*}
which is satisfied when
\begin{align*}
    \log \alpha &\leq \min_{N\geq3}\frac{(1-2N)}{N(N-1)} = \min_{N\geq3} \left(\frac{1}{1-N}-\frac{1}{N}\right) = -\frac56\\
    \alpha^{(N-2)/2}&\leq 1/2.
\end{align*}
We can further derive that when $\alpha\leq \exp(-5/6)\wedge 1/4$, we have a lower bound of $T_2-T_1$ is increasing as $N$ increases.

To show $T_4-T_3$ is increasing as $N$ increases, we just need to show $T_4$ is increasing. It follows a similar proof.

We can further derive that when $\alpha\leq \exp(-5/3)\wedge2$, we have $T_4-T_3$ is increasing as $N$ increases.

\end{proof}

\revise{
\subsection{Proof of Remark \ref{remark:init-early-stopping}}
\label{sec:init-early-stopping}
The proof is indeed similar to that of Theorem \ref{thm: early stopping}. Fixing any $N>2$ and step size $\eta$, we look at $T_2-T_1$ and $T4-T3$ and show that a tight lower bound of that is increasing as $\alpha$ decreases. We start with $T_2-T_1$.
}

\revise{
Recall that 
\begin{align*}
T_2 - T_1 
&= 
\frac{5}{N(N-1)\eta\zeta}\left(\frac{1}{\alpha^{(N-2)}}-\frac{1}{\alpha^{(N-2)/2}}\right)
-
\frac{15}{8N(N-2)\eta \zeta \alpha^{(N-2)}}) \\
&-
\frac{75}{16\eta N^2 \zeta^{(2N-2)/N}} \log \frac{|\wmax - \alpha^N|}{\epsilon}\\
&\geq
\frac{5}{N(N-1)\eta\zeta}\left(\frac{1}{\alpha^{(N-2)}}-\frac{1}{\alpha^{(N-2)/2}}\right)
-
\frac{75}{16\eta N^2 \zeta^{(2N-2)/N}} \log \frac{|\wmax|}{\epsilon}
\end{align*}
Notice that the second term is not about $\alpha$. We just need to show that 
$
f(\alpha) = \frac{1}{\alpha^{(N-2)}}-\frac{1}{\alpha^{(N-2)/2}}
$
is increasing as $\alpha$ decreases.
With the general requirement of $\alpha<1$, we have that
\begin{align*}
  f'(\alpha)
  &=  -\frac{(N-2)}{\alpha^{(N-1)}}+\frac{(N-2)/2}{\alpha^{N/2}}\\
  &= (N-2) \left(\frac{1}{2\alpha^{N/2}} - \frac{1}{\alpha^{(N-1)}}\right)\\
  &= (N-2) \frac{\alpha^{(N-2)/2}-2}{2\alpha^{(N-1)}}
  < 0. 
\end{align*}
For $T_4-T_3$, it follows a similar proof.
}

\revise{
\section{Experiments on MNIST}
\label{sec:more-experiments}
\begin{figure}[ht!]
    \centering
    \includegraphics[width=.8\linewidth]{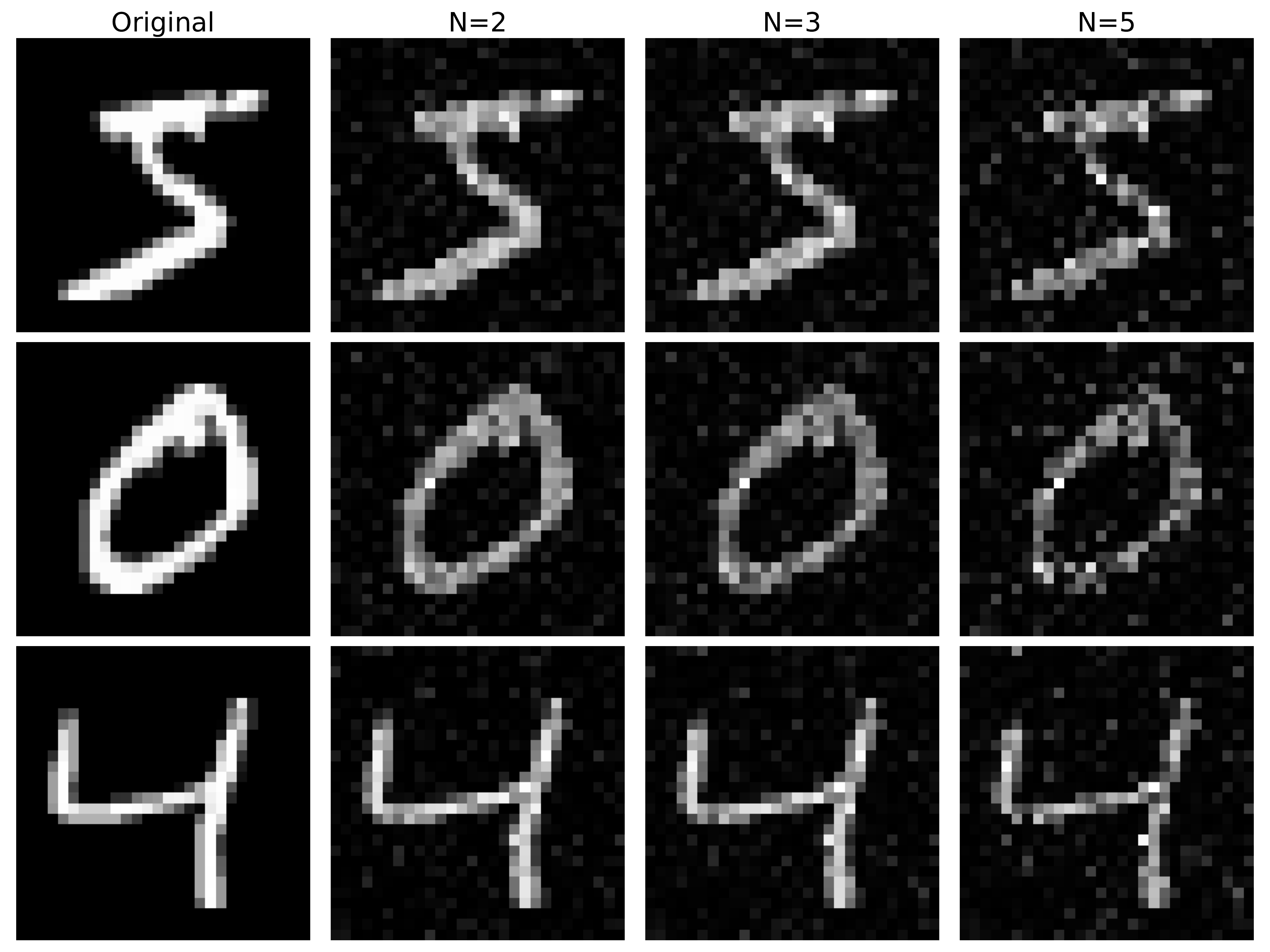}
    \caption{Experiments with different choice depth parameter $N$. The number of measurements is set as $n=392$, where the dimension of the original image is $p=784$. We use Rademacher sensing matrix. }
    \label{fig:mnist}
\end{figure}
The efficacy of different depth parameter $N$ is shown in Figure \ref{fig:convergence} and Figure \ref{fig:mnist} on both simulated data and real world datasets. The MNIST examples are successfully recovered from Rademacher linear measurements using different deep parametrizations.
}

\end{document}